\documentclass{article}

\usepackage[main, final]{neurips_2026}

\usepackage[utf8]{inputenc} 
\usepackage[T1]{fontenc}    
\usepackage{hyperref}       
\usepackage{url}            
\usepackage{booktabs}       
\usepackage{amsmath}        
\usepackage{amsfonts}       
\usepackage{amssymb}        
\usepackage{amsthm}         
\newtheorem{proposition}{Proposition}
\usepackage{nicefrac}       
\usepackage{microtype}      
\usepackage{xcolor}         
\usepackage{graphicx}       
\usepackage{subcaption}     
\usepackage{wrapfig}        
\usepackage{multirow}       

\title{Synthesis Through Simulation: Generating Coherent Enterprise Data via Scalable Agent-System Interaction}

\author{%
  \textbf{Yipeng Li\thanks{Equal contribution.}} \hspace{2em}
  \textbf{Ashutosh Hathidara\footnotemark[1]} \hspace{2em}
  \textbf{Jane Lo} \\[0.5em]
  \textbf{Harshavardhan Abichandani} \hspace{2em}
  \textbf{Gunraj Singh\thanks{Work done as part of an internship at SAP.}} \hspace{2em}
  \textbf{Atin Ghosh} \\[0.8em]
  \normalfont 
  SAP Labs \\
  \texttt{\{yipeng.li, ashutosh.hathidara, jane.lo\}@sap.com} \\
  \texttt{\{harshavardhan.abichandani, atin.ghosh\}@sap.com} \\
  \texttt{gunraj262005@gmail.com} \\
}

\begin{document}

\maketitle

\begin{abstract}
Tool-calling agents have become central to enterprise AI, yet training and evaluating them at scale remains severely constrained due to business and legal restrictions on enterprise systems, data, and database schemas. Tabular data synthesis offers a natural alternative, but its effectiveness is fundamentally limited by structural validity and schema availability, while procedure-based approaches yield the opposite weakness, typically lacking distributional fidelity without per-domain authoring. We introduce \textbf{Synthesis Through Simulation} (STS), a \textbf{schema--free} data synthesis paradigm in which an LLM agent generates data by executing operations against policy-enforcing APIs within simulated enterprise environments. Because data is generated through the same environment that defines what is valid, STS guarantees structural validity by construction while decoupling validity enforcement from distribution modeling, allowing each to be addressed independently. The \textbf{Generalist Populator} (GP), STS's domain-agnostic agent, addresses the remaining challenges of distributional fidelity and synthesis scalability: GP achieves \textbf{0.88 }average marginal fidelity and \textbf{100\% constraint satisfaction} across all ten environments \emph{without access to DB schemas}, while statistical synthesizers are inapplicable to seven due to necessary seed data requirements, and schema-privileged agents fail 82\% of trajectories on airline environment's tightly coupled workflows due to brittle task composition. We open-source the full framework, all ten environments, and generated datasets at \url{https://github.com/SAP/synthesis-through-simulation}.

\end{abstract}


\section{Introduction}
\label{sec:intro}

Tool-calling agents have become central to enterprise AI, enabling large language models (LLMs) not only to generate text but also take actions against existing 
systems~\cite{Yao2023react,Qin2024toolllm,Yao2024taubench}. Training and evaluating such agents at scale requires corpora of realistic, coherent
enterprise data in the form of pre-populated database snapshots that agents can act on without encountering invalid states. Yet assembling this data from real systems is severely constrained by privacy and compliance requirements; production enterprise databases often contain personal and sensitive data that cannot be extracted, shared, or used directly for agent training~\cite{Dwork2006,Abadi2016,Godbole2025}. Database schemas are often equally restricted, encoding proprietary domain models and system architectures that organizations are unwilling to expose. Synthetically generated data offers a natural alternative, but what does ``realistic'' mean for enterprise data, and how effective are these existing approaches for synthesis?

In statistical data synthesis, a record is treated as realistic if its feature values match the empirical distribution of a training corpus~\cite{Patki2016sdv,Stoian2025survey}. Enterprise validity is fundamentally different: it is not an intrinsic property of individual records but a \emph{relational} one, defined by the policies governing them. For example, an employee record must reference a valid department and pay-grade code, and a time-off request must not exceed the employee's accrued balance. These constraints encode business logic and are enforced at the API layer of production enterprise systems --- not as annotations on individual rows, but as preconditions checked against live database states at write time. 

Statistical tabular data generators such as CTGAN~\cite{Xu2019}, TabDDPM~\cite{Kotelnikov2023tabddpm}, and GReaT~\cite{Borisov2023great} optimise distributional fidelity but treat rows independently, have no mechanism for enforcing relational, policy-defined constraints, and require large pre-populated training corpora, creating a cold-start problem in enterprise deployments where production data cannot be exported due to privacy regulations~\cite{Dwork2006,Abadi2016}. In contrast, procedure-based synthesis encodes business rules directly into the generation process~\cite{Ge2021kamino,Li2023irg}, addressing structural validity but not \emph{distributional fidelity}: it specifies what is allowed, not what is typical. Additionally, procedure-based approaches typically request substantial \emph{per-environment authoring} in the form of constraint encodings, workflow scripts, and sampling rules that must be hand-crafted for each new enterprise domain, imposing a human effort cost that scales poorly as the number of environments grows.

We introduce \textbf{Synthesis Through Simulation} (STS), a data synthesis paradigm in which an LLM agent generates data by executing operations against policy-enforcing APIs within simulated enterprise environments. The core design principle is simple: because data is generated \emph{through} the same enforcement layer that defines what is valid, structural validity is guaranteed by construction, as any trajectory accepted by the API produces a coherent database snapshot. This also \emph{decouples} the two requirements; as validity enforcement is delegated entirely to the API, the agent is free to focus on distributional realism.

With structural validity guaranteed by the environment, we introduce the \textbf{Generalist Populator} (GP), STS's domain-agnostic agent that addresses the remaining challenges of distributional fidelity and scalability without requiring per-environment configuration. Given only the raw tool registry (the function signatures and descriptions exposed at the
API boundary), GP achieves validation pass rate, \(\mathrm{VPR} = 1\) across all ten environments and marginal fidelity above 0.93 in five, \textit{without} any access to the underlying database schema, hand-authored intent library, sampler, or per-environment configuration.

To summarize, our contributions include: \textbf{(i)} a formal problem decomposition of enterprise data synthesis into three independent axes --- structural validity, distributional fidelity, and scalability --- showing that API-layer generation guarantees structural validity by construction; \textbf{(ii)} \textbf{Synthesis Through Simulation (STS)}, a paradigm that operationalizes this decomposition by decoupling validity enforcement from distribution modeling; and \textbf{(iii)} the \textbf{Generalist Populator (GP)}, a zero-authoring agent requiring only the raw tool registry.


\section{Related Work}
\label{sec:related}

\textbf{Tabular data synthesis.}
Statistical approaches to tabular synthesis span GAN-based, diffusion-based, and LLM-based generators~\cite{Stoian2025survey}. CTGAN~\cite{Xu2019} uses conditional GANs with mode-specific normalisation; TabDDPM~\cite{Kotelnikov2023tabddpm} applies diffusion models to mixed tabular schemas; GReaT~\cite{Borisov2023great} fine-tunes a language model to generate rows as natural language. These approaches reproduce empirical distributions over a training corpus but treat each record independently, with no access to the enforcement layer that production systems rely on, and suffer from cold start problem if no seed data is available. The Synthetic Data Vault (SDV)~\cite{Patki2016sdv,Zhang2022sdv} addresses referential integrity between tables by conditioning child-table synthesis on parent foreign keys, yet cannot enforce procedural constraints that must be checked against live database state at write time. While Kamino~\cite{Ge2021kamino} and IRG~\cite{Li2023irg} handle static declarative constraints (functional dependencies, denial constraints, etc.) within constraint-aware synthesis frameworks, static constraint encodings remain fundamentally distinct from dynamic, API-enforced business logic that depends on current database state. In contrast, STS generates data \emph{through} the enforcement layer itself, requiring minimal reference data, making per-constraint encoding unnecessary, and guaranteeing structural validity by construction.

\textbf{Agentic data and environment synthesis.}
A growing body of work uses LLM agents to synthesize training data and evaluation environments. APIGen-MT~\cite{Prabhakar2025apigenmt} and DiaFORGE~\cite{hathidara2026disambiguationcentricfinetuningmakesenterprise} simulate agent-user interaction to generate multi-turn conversation trajectories for agent fine-tuning; ASTRA~\cite{Tian2026astra} synthesizes agentic trajectories and RL arenas via tool-call graph expansion. EnvScaler~\cite{Song2026envscaler} and ScaleEnv~\cite{Tu2026scaleenv} scale tool-interactive environments programmatically from seed specifications. These systems target agent training trajectories rather than enterprise database snapshots, and none enforce API-level business rules at write time, where structural validity is neither a design goal nor a measurement target.

\textbf{Tool-calling agent benchmarks.}
Tool-calling agents have been assessed through a range of evaluation benchmarks: AgentBench~\cite{Liu2023agentbench} and WebArena~\cite{Zhou2023webarena} span web and code environments; SWE-bench~\cite{Jimenez2024swebench} targets software engineering tasks; $\tau$-bench~\cite{Yao2024taubench} evaluates agents on policy-constrained retail and airline tasks using an LLM judge that verifies final states against natural-language policies; TED~\cite{chong2026talk} extends agent evaluation with user-aware interaction and automated error analysis. ReAct~\cite{Yao2023react}, Toolformer~\cite{Schick2023toolformer}, and ToolLLM~\cite{Qin2024toolllm} study agent capability for tool use at scale. These benchmarks characterise what agents can do but do not attempt to produce enterprise data as output. The most related is $\tau$-bench, which shares our enterprise-domain focus and policy-constrained setting; STS diverges by enforcing constraints programmatically at the API layer rather than via a post-hoc LLM judge, and by treating the populated database snapshot, not the agent's pass/fail score, as the primary deliverable. Table~\ref{tab:related-comparison} summarizes how these systems compare across the axes introduced in this paper.

\begin{table}[t]
  \centering
  \caption{Comparison of related systems. \checkmark~= satisfied; $\times$~= not; $\sim$~= partial.}
  \label{tab:related-comparison}
  \footnotesize
  \begin{tabular}{lccccccc}
    \toprule
    System & Output & Struct.\ validity & Dist.\ fidelity & Ref.\ dist.\ aware & Zero authoring & Schema-free \\
    \midrule
    CTGAN~\cite{Xu2019}                        & DB records   & $\times$   & \checkmark & \checkmark & $\times$ & $\times$ \\
    IRG~\cite{Li2023irg}                       & DB records   & $\sim$     & $\sim$     & $\times$   & $\times$ & $\times$ \\
    EnvScaler~\cite{Song2026envscaler}         & Envs + traj. & $\sim$     & \checkmark & $\times$   & $\times$ & $\times$ \\
    ScaleEnv~\cite{Tu2026scaleenv}             & Envs + traj. & $\sim$     & $\times$   & $\times$   & $\times$ & $\times$ \\
    $\tau$-bench~\cite{Yao2024taubench}        & Evaluation   & $\sim$     & $\times$   & $\times$   & $\times$ & \checkmark \\
    \midrule
    \textbf{STS (ours)}                        & \textbf{DB + traj.} & \checkmark & \checkmark & \checkmark & \checkmark & \checkmark \\
    \bottomrule
  \end{tabular}
\end{table}

\section{Problem Formulation}
\label{sec:problem}

\subsection{Enterprise Environments}

An \emph{enterprise environment} $\mathcal{E} = (\mathcal{D}, \mathcal{T}, \Psi, \mathcal{V})$ consists of: a database state space $\mathcal{D}$; a typed \emph{tool registry} $\mathcal{T}$, where each tool $t \in \mathcal{T}$ carries an input schema $\sigma_t$ and a precondition $\phi_t : \mathcal{D} \times \mathrm{Args} \to \{0,1\}$ encoding business rules that must hold at write time; a transition function $\Psi$, where each tool call $a = (t, x)$ is mapped as
\begin{equation}
  \Psi(a, d) =
  \begin{cases}
    d' & \text{if } \phi_t(d, x) = 1 \\
    d  & \text{otherwise.}
  \end{cases}
  \label{eq:psi}
\end{equation}
where $d' \triangleq \mathrm{apply}(t, x, d)$ is the deterministic state obtained by committing the write of tool $t$ with arguments $x$ to state $d$; and a validation suite $\mathcal{V} = \{r_1, \ldots, r_M\}$ of $M$ programmatic business-rule checks. 
The tool registry is the \emph{only interface} to the environment: the underlying database schema is not accessible, reflecting the enterprise reality that APIs are the sanctioned external boundary while schemas are subject to data-governance controls. A \emph{synthesis trajectory} $\tau = (a_1, \ldots, a_n)$ is a sequence of tool calls; starting from an initial state $d_0 = d_{\text{seed}}$ (where $d_{\text{seed}}$ may be empty or contain pre-populated seed records), the induced state sequence satisfies $d_i = \Psi(a_i, d_{i-1})$.

\subsection{Synthesis Objectives}

The \emph{structural validity} of a synthesizer is its Validation Pass Rate (VPR) over the terminal state $d_n$:
\begin{equation}
  \mathrm{VPR} = \frac{1}{M}\sum_{k=1}^{M}
    \mathbf{1}\bigl[r_k(d_n) = \textsc{pass}\bigr].
  \label{eq:vpr}
\end{equation}

\begin{proposition}[Structural Validity by Construction]
\label{prop:vpr}
For any trajectory $\tau$ executed through $\mathcal{E}$, $\mathrm{VPR} = 1$.
\end{proposition}
\begin{proof}
By induction. \emph{Base case:} $d_{\mathrm{seed}}$ is valid by assumption ($\varnothing$ satisfies all checks trivially). \emph{Inductive step:} a write is committed only when $\phi_{t_i}(d_{i-1},x_i)=1$, which by definition preserves every invariant in $\mathcal{V}$; rejected calls leave $d_{i-1}$ unchanged.
\end{proof}

Because VPR$\,{=}\,1$ is guaranteed by the environment, the synthesis challenge reduces to \emph{distributional fidelity and scalability}: generating value distributions that match real-world expectations, without per-environment intent authoring. We quantify this as \emph{distributional fidelity} (DF) across three levels of granularity --- marginal per-column (MDF), pairwise joint within a table (PDF), and cross-table relational child-count distributions (CDF) --- with formal metric definitions in Section~\ref{sec:experiments}. The synthesis task is to maximize DF, with VPR$\,{=}\,1$ guaranteed by construction. \textbf{This decoupling is the central design principle of STS}: validity is delegated entirely to the API enforcement layer, freeing the synthesizer to focus on distributional realism without constraint-encoding overhead.



\section{The STS Framework}
\label{sec:framework}

STS comprises two components: a suite of policy-enforcing environments that guarantee structural validity by construction, and the Generalist Populator (GP), a zero-authoring agent that synthesizes distribution-faithful data from the tool registry alone.

\begin{figure}[t]
  \centering
  \includegraphics[width=0.90\linewidth, trim = 0cm 0.5cm 0cm 0.5cm, clip]{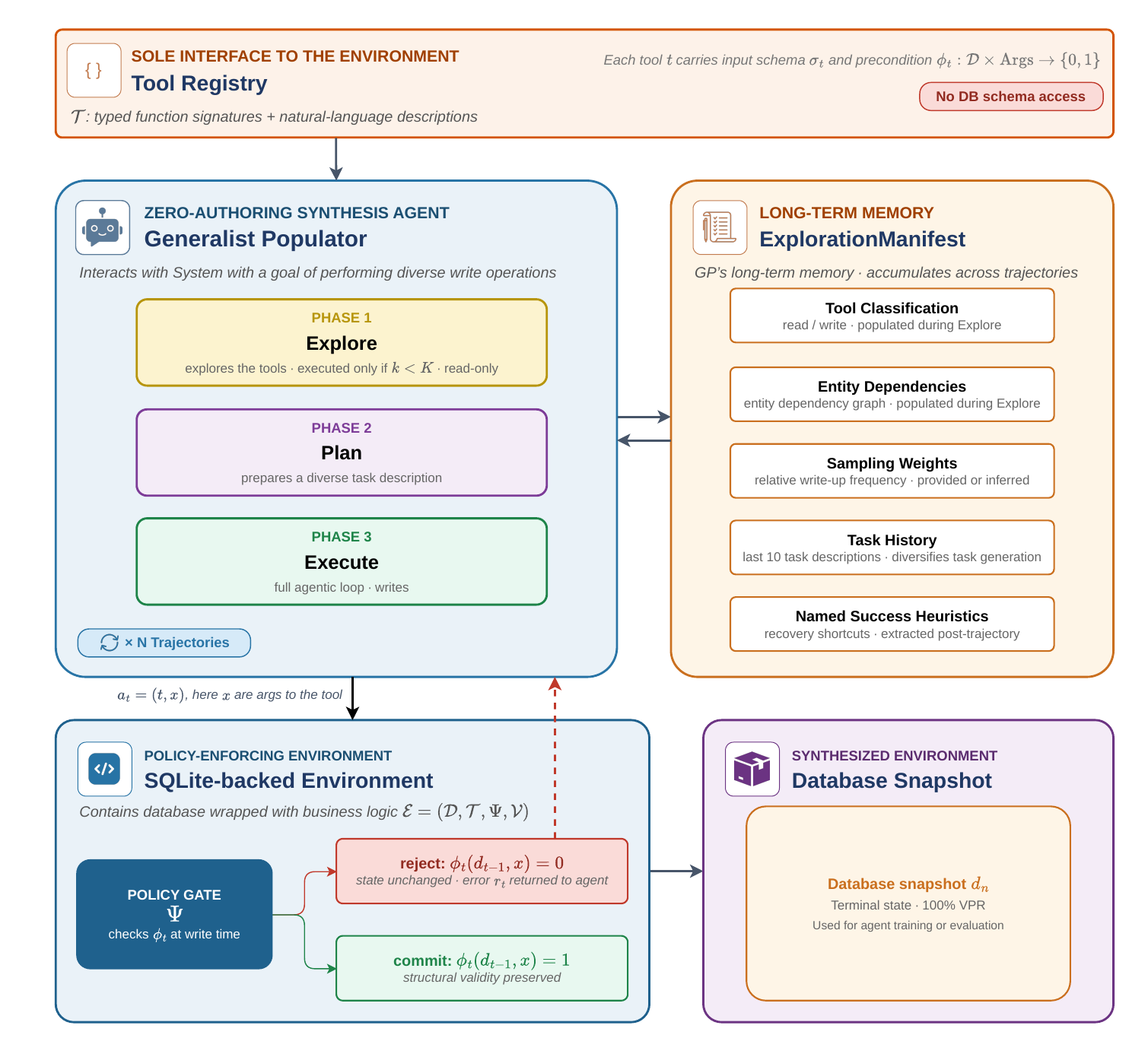}
  \caption{Overview of the STS framework. GP operates as an outer loop of $N$ trajectories,
           each passing through three phases (\textsc{Explore}, \textsc{Plan},
           \textsc{Execute}) against a policy-enforcing environment. The
           \texttt{ExplorationManifest} accumulates structural knowledge, distribution
           weights, and reusable heuristics shared across all trajectories.}
  \label{fig:pipeline}
\end{figure}

\subsection{Policy-Enforcing Environments}

The STS framework operates across ten enterprise environments spanning human resources, airline reservations, IT service management, banking, e-commerce, retail, university enrollment, pharmacy dispensing, online forums, and cloud computing infrastructure. Each environment consists of a SQLite-backed mock enterprise system that exposes the typed tool registry $\mathcal{T}$ and validation suite $\mathcal{V}$ formalized in Section~\ref{sec:problem}. The tool registry is the sole interface: a set of typed function signatures --- name, natural-language description, and input schema --- that encapsulate all state-mutating operations. No database schema is exposed by design; in many enterprise deployments, schema access requires the same review as production data access. Table~\ref{tab:environments} in Appendix \ref{app:environments} summarizes the environments by domain, table count, tool count, and number of programmatic validation checks.

\subsection{The Generalist Populator}

The \textbf{Generalist Populator} (GP) is a zero-authoring, domain-agnostic synthesis agent requiring only the raw tool registry, with no access to the underlying database schema, no hand-authored task specifications, and no per-environment configuration. Optional reference distributions may be supplied, but are not used in our evaluation for fair comparison. Rather than generating records directly, GP populates the database exclusively by issuing API calls through the tool registry, so every write is mediated by the environment's enforcement layer, and structural validity is guaranteed by construction. GP operates as an outer loop of $N$ trajectories, where each trajectory executes one concrete task by interacting with the environment through its API. Running GP $N$ times thus produces $N$ tasks worth of database population. Across all trajectories, GP maintains a persistent \texttt{ExplorationManifest} --- a lightweight SQLite store that serves as long-term memory, accumulating structural knowledge, distribution weights, task history, and reusable heuristics that all future trajectories can draw on. Each trajectory passes through three phases (\textsc{Explore}, \textsc{Plan}, and \textsc {Execute}), with information from the manifest provided as context at every
stage.

\textbf{\textsc{Explore} phase.}
The goal of \textsc{Explore} is to build a mental model of the environment sufficient for goal-directed synthesis: which tools exist, whether the tool is a read or write operation, what entities they create, in what order entities must be created, and how frequently each operation should appear in a realistic workload, including inferring plausible write-operation sampling weights from world knowledge when no reference distributions are provided. Because this model is stable once acquired, \textsc{Explore} runs only for the first $K$ trajectories and is skipped thereafter, with \textsc{Plan} and \textsc{Execute} reusing the manifest directly. 

During exploration, all write tools are blocked with a structured rejection message so the agent builds its model purely from read calls and tool signature inspection, without modifying the database. This write guard is deliberate: premature writes made under an incomplete model would introduce poorly-targeted records that degrade distributional fidelity before the agent has sufficient context.  

The agent records its findings into the manifest via dedicated manifest tools: it classifies every tool as read-only or state-mutating, identifies entity types and the tools that create them, infers entity dependency ordering from argument schemas, and estimates relative write-operation sampling weights (when not already provided) that reflect how frequently each operation should appear.

\begin{figure}[t]
  \centering
  \begin{subfigure}{0.32\linewidth}
    \includegraphics[width=\linewidth]{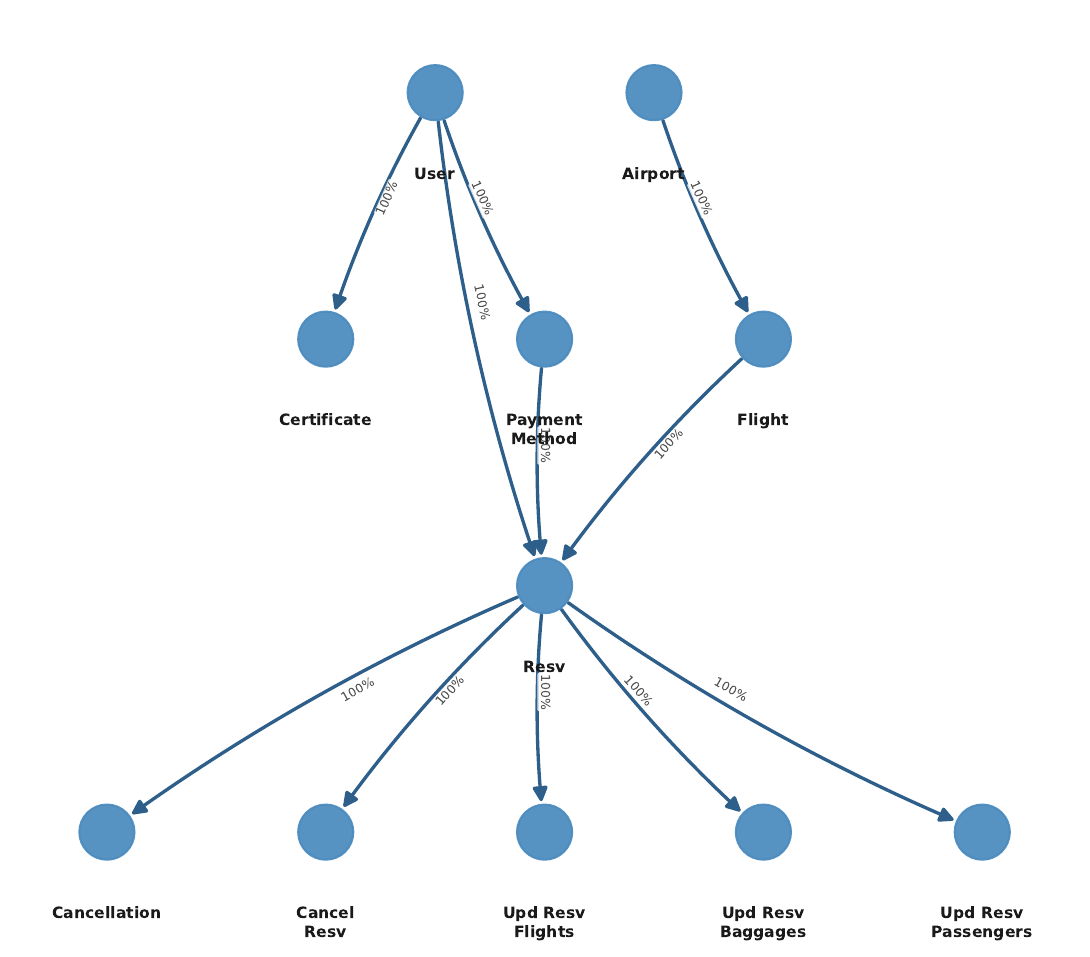}
    \caption{Airline (16 tools, 11 entities)}
    \label{fig:manifest-airline}
  \end{subfigure}
  \hfill
  \begin{subfigure}{0.32\linewidth}
    \includegraphics[width=\linewidth]{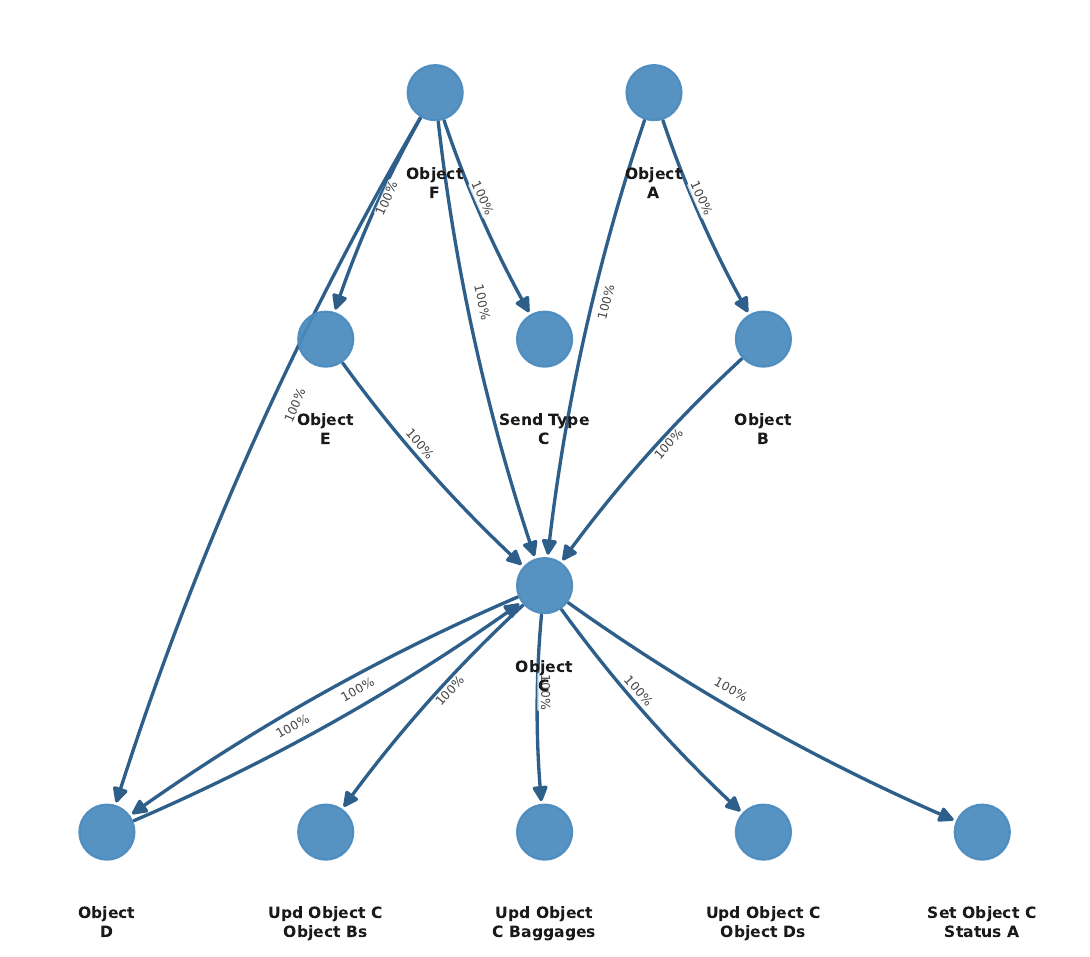}
    \caption{Obfuscated Airline}
    \label{fig:manifest-obfuscated-airline}
  \end{subfigure}
  \hfill
  \begin{subfigure}{0.32\linewidth}
    \includegraphics[width=\linewidth]{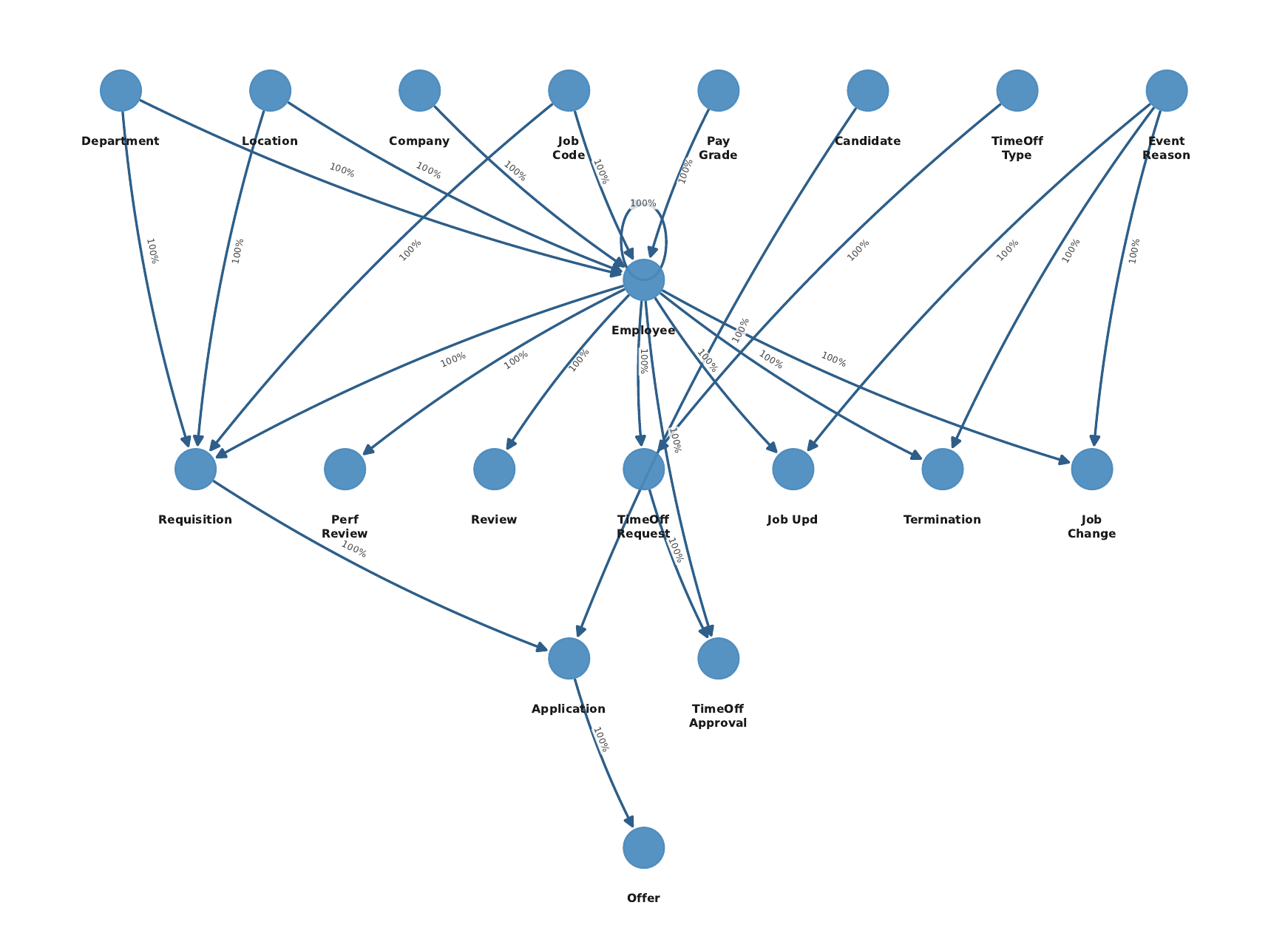}
    \caption{HR (30 tools, 19 entities)}
    \label{fig:manifest-hr}
  \end{subfigure}
  \caption{\textbf{Entity dependency graphs dynamically inferred by GP during the \textsc{Explore} phase}
           for three environments.
           Nodes are entity types; directed edges represent creation dependencies
           inferred from tool argument schemas.}
  \label{fig:manifests}
\end{figure}

\textbf{\textsc{Plan} phase.}
Given the manifest, Plan issues a single structured LLM call to produce a concrete \emph{task description} for the current trajectory. The call receives the full manifest context (entity dependency graph, tool classifications, operation weights) together with two additional signals that promote distributional coverage across trajectories: recent task history (so the planner avoids repeating prior tasks) and a coverage check that flags unused write operations and requires the planner to prioritize them.

To guide attribute-level distributional fidelity, Plan employs a two-tier mechanism: when per-environment reference distributions are provided (e.g., target proportions for categorical fields), the planner anchors its attribute targets to these priors; when no reference is available, it draws on the world knowledge of the underlying language model to infer plausible value distributions. The output is a \texttt{task\_description}: a concrete, multi-step specification with target attribute values (e.g., ``book a round-trip business class flight for a gold-tier user'') but intentionally generic entity references, since specific identifiers such as user IDs or reservation numbers must be discovered from the live database at execution time.

\textbf{\textsc{Execute} phase.}
Execute instantiates the task description through a full agentic loop with unrestricted access to the complete tool registry. The agent's system prompt is assembled from three sources: manifest context, the task description from \textsc{Plan}, and accumulated heuristics. The agent queries the live database using certain read operations for the entity identifiers it needs, then calls write tools in dependency order, handling rejected calls (those for which $\phi_t(d, x) = 0$) by adapting from the structured error response. After the trajectory completes, GP performs a best-effort heuristic extraction step: an LLM reviews the full tool-call sequence and identifies whether the agent discovered a shortcut late in the trajectory, a call that, once made, unlocked a sequence of subsequent successes after many failed attempts (e.g., listing existing reservations to retrieve valid flight routes rather than guessing route arguments exhaustively). If such a pattern is found, it is persisted in the manifest as a named heuristic, so all future trajectories can benefit from the shortcut from the outset rather than rediscovering it through trial and error. 


\section{Experiments}
\label{sec:experiments}

We evaluate GP across all ten enterprise environments described in Section~\ref{sec:framework} (with full implementation details available in Appendix~\ref{app:impl} ). We compare it against four tabular data generation methods: CTGAN~\cite{Xu2019}, TVAE~\cite{Xu2019}, SDV HMA~\cite{Patki2016sdv,Zhang2022sdv}, and REaLTabFormer~\cite{Solatorio2023} (see configurations in Appendix~\ref{app:stat-config}). We additionally compare it to \emph{EnvScaler}~\cite{Song2026envscaler}, a schema-privileged agent that receives the full database schema, live state, and business rules, none of which are provided to (or required by) GP. In all experiments, the GP \textsc{Explore} phase is given no human-constructed write operation distributions, enabling it to infer priors from scratch, ensuring a fair comparison with baselines and preventing any leakage of human bias.

\subsection{Metrics}
\label{sec:metrics}

All three distributional metrics share a common JSD-based similarity foundation: $\mathrm{sim}(P,Q)=1-\mathrm{JSD}(P\|Q)/\ln 2 \in [0,1]$~\cite{Lin1991jsd}, where 1 means identical distributions and 0 means maximally divergent.
Each metric operates over a fixed set of \emph{registered} fields --- columns, column pairs, or FK relationships whose reference distributions are grounded in publicly available data (industry reports, government statistics), ensuring priors are fixed independently of any synthesizer (Appendix~\ref{app:priors}).
\emph{Marginal Distribution Fidelity} (MDF) averages per-column $\mathrm{sim}$ over all registered (table, column) pairs (Table~\ref{tab:mdf-priors}), comparing the agent's generated distribution against a reference marginal; numeric columns substitute the complement of the Kolmogorov--Smirnov statistic~\cite{Kolmogorov1933,Smirnov1948}.
\emph{Pairwise Distribution Fidelity} (PDF) applies the same similarity to \emph{joint} within-table distributions for registered column pairs (Table~\ref{tab:pdf-priors}), capturing correlations invisible to per-column statistics.
\emph{Cross-table Distribution Fidelity} (CDF) measures whether the distribution of child-record counts per agent-created parent matches reference expectations across each registered FK relationship (Table~\ref{tab:cdf-priors}).
\emph{Average Successful Write Operations} (WO) is the mean number of state-changing API calls successfully committed per trajectory.
Structural validity is reported as \emph{Validation Pass Rate} (VPR, Equation~\ref{eq:vpr}); by Proposition~\ref{prop:vpr}, VPR\,=\,1 is guaranteed for any STS trajectory.
Full formal derivations appear in Appendix~\ref{app:metrics}.
All distributional metrics report 95\% bootstrap CIs ($B\!=\!1000$).

\subsection{GP Performance Across All Environments}
\label{sec:df-results}

\begin{table}[t]
  \centering
  \caption{GP ($K\!=\!3$) distributional fidelity across ten environments and one obfuscated variant
           ($N\!=\!100$ trajectories each); 95\% bootstrap CIs.
           \emph{Records}: total new database rows generated.
           ---: no column pairs (PDF) or FK relationships (CDF) registered for that environment.}
  \label{tab:df-main}
  \footnotesize
  \begin{tabular}{lccccc}
    \toprule
    Environment & MDF & PDF & CDF & WO\,(mean\,$\pm$\,std) & Records \\
    \midrule
    Airline       & 0.973 [0.935, 1.000] & 0.900 [0.898, 0.902] & 0.949 [0.923, 0.975] & 1.07\,$\pm$\,0.74 & 389 \\
    Obfus. Airline & 0.977 [0.954, 0.995] & 0.933 [0.928, 0.937] & 0.985 [0.971, 0.998] & 1.37\,$\pm$\,1.43 & 518 \\
    IT Mgmt       & 0.966 [0.963, 0.969] & 0.902 [0.789, 0.938] & --- & 1.29\,$\pm$\,0.80 & 54 \\
    Forum         & 0.945 [0.945, 0.945] & --- & 0.318 [0.142, 0.401] & 2.27\,$\pm$\,2.27 & 257 \\
    HR            & 0.939 [0.939, 0.939] & --- & 0.505 [0.505, 0.505] & 3.32\,$\pm$\,2.60 & 756 \\
    Banking       & 0.864 [0.814, 0.913] & 0.859 [0.706, 0.940] & 0.421 [0.236, 0.516] & 4.59\,$\pm$\,2.83 & 486 \\
    Cloud         & 0.850 [0.707, 0.943] & 0.825 [0.765, 0.847] & 0.610 [0.311, 0.757] & 1.85\,$\pm$\,1.39 & 102 \\
    University    & 0.809 [0.771, 0.824] & --- & --- & 0.75\,$\pm$\,0.71 & 59 \\
    Pharmacy      & 0.804 [0.793, 0.804] & --- & --- & 1.80\,$\pm$\,2.34 & 90 \\
    Retail        & 0.709 [0.420, 0.997] & 0.993 [0.764, 0.997] & 0.643 [0.461, 0.826] & 5.58\,$\pm$\,2.96 & 490 \\
    E-commerce    & 0.983 [0.914, 0.999] & ---   & 0.578 [0.442, 0.665] & 2.56\,$\pm$\,1.27 & 189 \\
    \bottomrule
  \end{tabular}
\end{table}

GP achieves $\mathrm{VPR}\!=\!1.0$ across all ten environments, confirming Proposition~\ref{prop:vpr}.
Table~\ref{tab:df-main} reports distributional fidelity and write productivity. Environments Airline (MDF 0.973, PDF 0.900, CDF 0.949), IT Mgmt (MDF 0.966, PDF 0.902), Forum (MDF 0.945), HR (MDF 0.939), and E-commerce (MDF 0.983) achieve MDF above 0.93 across diverse domains.
HR and Banking have the highest mean WO (3.32 and 4.59), reflecting multi-step workflows that benefit from the exploration manifest.

Several environments show more moderate MDF: Banking (0.864), Cloud (0.850), University (0.809), and Pharmacy (0.804).
Two environments additionally show low CDF: Forum (0.318) and Banking (0.421), due to a \emph{task-ratio imbalance}.  GP's planner draws on LLM world knowledge to estimate task-type probabilities, which assigns roughly equal weight to entity-creation and follow-on operation intents without modeling population-level operational skew.
In Forum, thread-creation tasks occur at the same frequency as post-reply tasks, yielding mostly 1--2 posts per thread while the reference expects a median of 3 (tail to 10); in Banking, account-creation tasks are sampled at a similar rate to deposit and withdrawal tasks, while the reference expects 5--30 transactions accumulated per account over time.
Both gaps reflect a world-knowledge miscalibration for task-type frequencies in the absence of domain-specific frequency signals, not an inability to generate relational structure; automatic inference of realistic operation priors from seed database statistics or domain-context prompting is a natural future direction.

To test semantic independence, we evaluate the same configuration on an \emph{obfuscated airline} variant where all identifiers are replaced with opaque codes (e.g., \texttt{book\_object\_c}, \texttt{type\_b} for business class). GP matches the semantic baseline on all three metrics.
Figure~\ref{fig:manifests} shows the entity dependency graphs inferred by GP during \textsc{Explore} for both airline variants and HR; the semantic and obfuscated airline graphs (Figures~\ref{fig:manifest-airline} and~\ref{fig:manifest-obfuscated-airline}) share 11 nodes and differ by only 2 edges ($\mathrm{GED}_{\mathrm{norm}}\!=\!0.154$; Appendix~\ref{app:graph-similarity}), suggesting that GP can recover entity creation ordering from argument schema structure alone, without relying on domain-interpretable names. These results provide evidence that STS can operate in settings where semantic context is limited, a common phenomenon in enterprise deployments.

\subsection{Generalization to Real Enterprise Systems}
\label{sec:real-world}

Apart from the environments evaluated above, a central question is whether GP's zero-authoring design transfers to real enterprise systems. We evaluate GP on two real-world systems with no per-environment configuration and no schema access: a proprietary multi-module \textbf{Payroll} system (42 tools, 13 tables) and \textbf{MintHCM}~\cite{MintHCM}, an open-source Human Capital Management system (54 tools, 23 tables). Table~\ref{tab:real-world} reports the results.

\begin{table}[t]
  \centering
  \caption{Real-world generalization: GP ($K\!=\!3$, $N\!=\!100$) on two real enterprise systems.}
  \label{tab:real-world}
  \footnotesize
  \begin{tabular}{lccccc}
    \toprule
    System & MDF & PDF & CDF & WO\,(mean\,$\pm$\,std) & Records \\
    \midrule
    Payroll & 0.927 [0.870, 0.974] & 0.748 [0.603, 0.803] & 0.841 [0.796, 0.903] & 11.09\,$\pm$\,9.74 & 1012 \\
    MintHCM & 0.902 [0.827, 0.955] & 0.344 [0.249, 0.438] & 0.634 [0.542, 0.717] & 11.97\,$\pm$\,10.50 & 478 \\
    \bottomrule
  \end{tabular}
\end{table}

On both systems GP achieves $\mathrm{VPR}\!=\!1.0$ (by design) and marginal fidelity $\mathrm{MDF}\!\ge\!0.90$ without any schema access. MintHCM's lower PDF (0.344) is a data-volume effect, the joint distribution of 17 funnel stages $\times$ 7 sourcing channels in the candidatures table (119 possible pairs) requires more trajectories to cover adequately whereas N=100. These results provide evidence that STS can generalize to real-world systems.

\subsection{LLM Backbone Sensitivity}
\label{sec:backbone}

We run GP ($K\!=\!3$, $N\!=\!100$) with three distinct backbones to evaluate its sensitivity to the choice of LLM: \textbf{Claude Opus 4.5}~\citep{anthropic2025claudeopus45}, \textbf{Gemini 2.5 Pro}~\citep{comanici2025gemini25}, and \textbf{Qwen3.6-35B-A3B-FP8}~\citep{qwen36_35b_a3b}. We perform experiments across three environments (Airline, HR, Cloud). Table~\ref{tab:backbone} reports the results.

\begin{wraptable}{r}{0.65\linewidth}
  \vspace{-10pt}
  \centering
  \setlength{\tabcolsep}{2pt}
  \caption{Backbone sensitivity: GP ($K\!=\!3$, $N\!=\!100$) with three LLM backbones across three environments. CIs in Appendix~\ref{app:backbone}.}
  \label{tab:backbone}
  \footnotesize
  \begin{tabular}{llcccc}
    \toprule
    Environment & Backbone & MDF & PDF & CDF & WO\,(mean\,$\pm$\,std) \\
    \midrule
    \multirow{3}{*}{Airline}
      & Claude Opus 4.5     & 0.973 & 0.900 & 0.949 & 1.07\,$\pm$\,0.74 \\
      & Gemini 2.5 Pro      & 0.937 & 0.685 & 0.957 & 0.68\,$\pm$\,0.83 \\
      & Qwen3.6-35B-A3B-FP8 & 0.904 & 0.788 & 0.980 & 2.05\,$\pm$\,1.84 \\
    \midrule
    \multirow{3}{*}{HR}
      & Claude Opus 4.5     & 0.939 & --- & 0.505 & 3.32\,$\pm$\,2.60 \\
      & Gemini 2.5 Pro      & 0.940 & --- & 0.514 & 2.72\,$\pm$\,3.25 \\
      & Qwen3.6-35B-A3B-FP8 & 0.937 & --- & 0.509 & 8.91\,$\pm$\,11.65 \\
    \midrule
    \multirow{3}{*}{Cloud}
      & Claude Opus 4.5     & 0.850 & 0.825 & 0.610 & 1.85\,$\pm$\,1.39 \\
      & Gemini 2.5 Pro      & 0.831 & 0.838 & 0.793 & 10.60\,$\pm$\,48.29 \\
      & Qwen3.6-35B-A3B-FP8 & 0.939 & 0.837 & 0.504 & 2.55\,$\pm$\,4.71 \\
    \bottomrule
  \end{tabular}
  \vspace{-6pt}
\end{wraptable}

All three backbones maintain $\mathrm{VPR}\!=\!1.0$ with meaningful fidelity: no backbone causes GP to degenerate. Fidelity varies across backbones (e.g., Airline PDF 0.685--0.900, Cloud CDF 0.504--0.793), but the variation is comparable in magnitude to the cross-environment spread already present in Table~\ref{tab:df-main} and reflects differences in model pretraining rather than a structural fragility of STS. The Gemini Cloud WO mean (10.60\,$\pm$\,48.29) is inflated by a few long exploration trajectories, consistent with high-variance exploration rather than instability.

\subsection{Coverage Advantage: GP vs.\ Tabular Data Synthesizers}
\label{sec:statistical-baselines}

Statistical synthesizers suffer from cold starts, as seven of ten environments have no pre-populated seed data. These are not edge cases but fresh enterprise deployment-like scenarios where the synthesis
task is to \emph{create} the initial database state from scratch. Figure~\ref{fig:baselines} compares GP against CTGAN~\cite{Xu2019}, TVAE~\cite{Xu2019},
HMA~\cite{Patki2016sdv,Zhang2022sdv}, and REaLTabFormer~\cite{Solatorio2023} on
the three environments where seed data exists. We observe that the primary differentiator is structural validity: GP achieves $\mathrm{VPR}\!=\!1.0$ by construction, while offline methods reach at most 0.923 --- a ceiling that no amount of additional training data can overcome, as these methods bypass the API enforcement layer entirely. On airline, GP also leads on distributional fidelity, but falls slightly below on retail and banking environments.

GP's throughput (per-trajectory latency 41-118s, full latency-cost analysis in Appendix~\ref{app:throughput}) is considerably lower than offline synthesizers on a per-row basis; however, this comparison conflates distinct use cases. Offline synthesizers require seed data that is absent in enterprise use cases, and using production PII to train them creates a compliance double-bind~\cite{Dwork2006,Abadi2016}.

Ablating the enforcement loop entirely, feeding the same backbone (Claude Opus 4.5) the raw schema and seed rows and inserting its output directly into the database, bypassing the API/policy layer, drops VPR to 0.846/0.820 on airline/HR. This confirms that valid-by-construction data comes from the simulation loop rather than the LLM alone (full table and analysis in Appendix~\ref{app:llm-direct}).


\begin{figure}[t]
  \centering
  \includegraphics[width=0.75\linewidth]{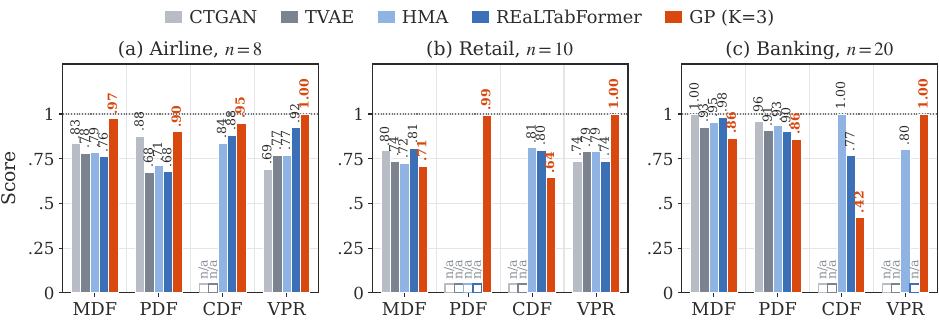}
  \caption{Statistical synthesizers vs.\ GP ($K\!=\!3$) on the three environments
           where seed data exists. Single-table methods (CTGAN, TVAE) produce no CDF by design.}
  \label{fig:baselines}
\end{figure}


\begin{figure}[t]
  \begin{minipage}[b]{0.65\linewidth}
    \centering
    \scriptsize
    \setlength{\tabcolsep}{3pt}
    \begin{tabular}{llccc}
      \toprule
      Env. & System & MDF & PDF & CDF \\
      \midrule
      Airline
        & GP $K\!=\!3$ & \textbf{0.973} [0.935, 1.000] & \textbf{0.900} [0.898, 0.902] & \textbf{0.949} [0.923, 0.975] \\
        & EnvScaler    & n/a & n/a & n/a \\
      \midrule
      Cloud
        & GP $K\!=\!3$ & \textbf{0.850} [0.707, 0.943] & \textbf{0.825} [0.765, 0.847] & 0.610 [0.311, 0.757] \\
        & EnvScaler    & 0.828 [0.644, 0.939] & 0.821 [0.773, 0.839] & \textbf{0.649} [0.529, 0.729] \\
      \midrule
      HR
        & GP $K\!=\!3$ & 0.939 [0.939, 0.939] & --- & 0.505 [0.505, 0.505] \\
        & EnvScaler    & 0.939 [0.939, 0.939] & --- & \textbf{0.509} [0.509, 0.509] \\
      \bottomrule
    \end{tabular}
    \vspace{3pt}
    \captionof{table}{GP ($K\!=\!3$) vs.\ EnvScaler ($N\!=\!100$); 95\% bootstrap CIs.
               \textbf{n/a}: EnvScaler generated no rows (82\% zero-write trajectories).
               ---: no PDF column pair registered.}
    \label{tab:envscaler}
  \end{minipage}%
  \hfill
  \begin{minipage}[b]{0.32\linewidth}
    \centering
    \includegraphics[width=\linewidth]{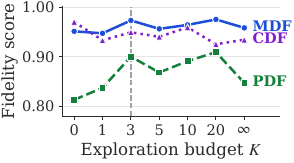}
    \captionof{figure}{Distributional fidelity vs.\ exploration budget $K$ on airline
               ($N\!=\!100$ per condition). $K$ values are evenly spaced for readability.}
    \label{fig:k-ablation}
  \end{minipage}
\end{figure}

\subsection{GP vs.\ Schema-Privileged Synthesis}
\label{sec:envscaler}

We explore the impact of GP's drastically reduced requirements on seed data, schemas, and per-domain authoring on data quality; Table~\ref{tab:envscaler} compares GP ($K\!=\!3$) against EnvScaler on three
environments (airline, cloud, HR; $N\!=\!100$). EnvScaler~\cite{Song2026envscaler} receives the full database schema, live state, and business rules, while GP receives only API tool descriptions.


On airline, EnvScaler produces almost no data (82\% zero-write trajectories) due to two distinct causes (example in Appendix~\ref{app:envscaler-example}): (i) a \emph{stale-ID collision artifact}: schema-grounded task descriptions embed specific entity IDs at task-generation time that become stale by execution time, causing hard failures on the first write call; and (ii) the genuine operational brittleness of entity-specific task planning against evolving shared state. We attribute the bulk of the airline failures to the former artifact and do not claim the full 82\% as a capability gap. On cloud, CDF favours EnvScaler (0.649 vs.\ 0.610), as schema knowledge directly encodes the \texttt{vpcs\,$\to$\, subnets} FK hierarchy. On HR, distributional quality difference is minimal (MDF tied; CDF 0.505 vs.\ 0.509) despite EnvScaler's informational advantage. Such results suggest that schema access improves relational depth only where it directly encodes FK structure; it does not improve marginal or pairwise fidelity, and can actively harm execution. Notably, \textbf{GP achieves this without any access to the database schema, matching a fully schema-privileged baseline}.


\subsection{Ablation Study}
\label{sec:ablation}

\noindent\textbf{Exploration budget ($K$).}
Figure~\ref{fig:k-ablation} shows how MDF, PDF, and CDF evolve as the exploration budget $K$ increases on airline ($N=100$; full numerical results in Appendix~\ref{app:k-analysis}).
PDF shows the strongest sensitivity, rising sharply from $K=0$ (0.813) to $K=3$ (0.900) then flattening, suggesting that three exploration trajectories suffice to surface the primary entity relationships.
MDF remains largely flat (0.95--0.98 with peak at $K=3$), confirming that marginal structure is captured even without prior exploration. CDF remains flat through $K=10$, then dips at larger budgets as over-diversified exploration shifts child-count distributions away from the reference prior. $K=3$ is the practical sweet spot: it captures the full PDF gain with no systematic benefit from additional exploration.

\noindent\textbf{\textsc{Plan} phase.}
We disable \textsc{Plan} on airline ($K=3$, $N=50$), replacing the distribution-targeted task description with a generic populate directive (full analysis in Appendix~\ref{app:plan-ablation}).
All fidelity metrics drop (MDF $-$0.065, PDF $-$0.062, CDF $-$0.071) \textit{despite} the agent writing substantially more data per trajectory (16.7 bookings vs.\ 0.61 with planning).
The cause is distributional skew: without a scoped task the agent bulk-iterates over all users in sequence, systematically under-representing compositionally complex intents, round-trip proportion collapses from 52.5\% to 7.9\% against a 55\% reference prior.
Higher write volume does not compensate for distributional divergence; \textsc{Plan} is what steers each trajectory toward a specific intent and its associated attribute distribution, not merely toward maximum throughput.

\noindent\textbf{Heuristic extraction.}
As part of this, we disable heuristic extraction and injection in \textsc{Execute} phase for three environments ($K=3$, $N=100$; full results in Appendix~\ref{app:heuristic-ablation}).
Heuristics deliver value through three distinct mechanisms: \textbf{workflow productivity} in banking, where their absence reduces WO by 33\% and PDF by 0.175 as 69\% of trajectories collapse to a stereotyped three-step sequence; \textbf{operational reliability} in HR, where zero-write failure rises from 0.07 to 0.12 without heuristics preventing dead-end paths; and \textbf{trajectory efficiency} in airline, where Calls/WO nearly doubles ($+$96\%) without a ``list existing reservations'' shortcut.
Each environment exposes a different failure mode, confirming that heuristic extraction is load-bearing across diverse workflow types.

\subsection{Downstream Agent Training}
\label{sec:downstream}

One of the primary use cases of such a simulated environment is to use it as a realistic starting state for downstream agent training. In enterprise environments, using production environments for training is often restricted due to PII violation concerns. STS enables such agent training. To demonstrate this utility concretely, we test whether agent training performed on the simulated environment snapshot yields better-performing agents than data generated from the pre-populated seed data snapshot, evaluated on the $\tau^2$-bench~\cite{barres2025tau2} airline benchmark.

\begin{wraptable}{r}{0.62\linewidth}
  \vspace{-10pt}
  \centering
  \setlength{\tabcolsep}{2pt}
  \caption{TED evaluation~\cite{chong2026talk} on the $\tau^2$-bench~\cite{barres2025tau2} airline benchmark across four LLM backbones.}
  \label{tab:downstream}
  \footnotesize
  \begin{tabular}{llccc}
    \toprule
    Backbone & Condition & Mean Prog & Mean AUC & Mean PPT \\
    \midrule
    \multirow{3}{*}{Qwen2.5-3B-Instruct~\citep{qwen2.5}}
      & Base       & 0.165 & 0.120 & 0.062 \\
      & SFT (STS)  & \textbf{0.178} & \textbf{0.149} & \textbf{0.077} \\
      & SFT (Seed) & 0.164 & 0.010 & 0.036 \\
    \midrule
    \multirow{3}{*}{Qwen3.5-4B~\citep{qwen3.5}}
      & Base       & 0.790 & 0.680 & 0.284 \\
      & SFT (STS)  & \textbf{0.790} & \textbf{0.680} & \textbf{0.268} \\
      & SFT (Seed) & 0.753 & 0.630 & 0.250 \\
    \midrule
    \multirow{3}{*}{Qwen3.5-9B~\citep{qwen3.5}}
      & Base       & 0.830 & 0.720 & 0.260 \\
      & SFT (STS)  & \textbf{0.850} & \textbf{0.740} & \textbf{0.280} \\
      & SFT (Seed) & 0.810 & 0.630 & 0.250 \\
    \midrule
    \multirow{3}{*}{Gemma-4-E2B-it~\citep{gemmateam2026gemma4}}
      & Base       & 0.234 & 0.180 & 0.080 \\
      & SFT (STS)  & \textbf{0.281} & 0.230 & 0.104 \\
      & SFT (Seed) & 0.249 & \textbf{0.247} & \textbf{0.120} \\
    \bottomrule
  \end{tabular}
  \vspace{-6pt}
\end{wraptable}


\noindent\textbf{Setup.}
Using the airline environment, we generate multi-turn rollouts with GPT-4.1 following TED's dynamic user-simulation methodology~\cite{chong2026talk} from two environment snapshots: STS-populated and pre-populated seed. Sampling 20 tasks $\times$ 3 rollouts and retaining completions with progress=1.0 yields 19 SFT (STS) and 18 SFT (Seed) training samples for fine-tuning four LLM backbones, evaluated on the $\tau^2$-bench airline domain.

\noindent\textbf{Results.}
Table~\ref{tab:downstream} shows that SFT (STS) outperforms SFT (Seed) in \textbf{10 of 12 model--metric pairs (${\sim}83\%$)}. The two exceptions are both on Gemma-4-E2B-it~\citep{gemmateam2026gemma4} --- Mean AUC ($-0.017$) and Mean PPT ($-0.016$) --- with margins $\le 0.017$; on every other backbone--metric combination STS wins outright. The consistency of this advantage across four backbones of different families and scales is consistent with the richer, more diverse database states produced by STS translating into improved agent capability.

\section{Conclusion}
\label{sec:conclusion}

We introduced STS, a paradigm guaranteeing structural validity by construction, and GP, a zero-authoring agent requiring no schema access or per-environment configuration.
Across ten enterprise environments, STS achieves VPR\,=\,1.0 by construction across all
synthesizer variants, and GP matches or exceeds the schema-privileged EnvScaler baseline
despite a strictly more constrained information budget. Agents fine-tuned on STS-populated snapshots outperform those trained on pre-populated seed snapshots in 10 of 12 model--metric pairs across four backbones, supporting the downstream utility of STS snapshots.

\noindent\textbf{Limitations \& Future Work.}
GP's fidelity can be sensitive to model quality and the two-tier distribution
mechanism introduces model-specific biases for unanchored fields.
STS's validity guarantee is conditional on the API enforcement layer being complete with respect to the environment's business rules; deploying STS against a real enterprise system requires that all relevant constraints be encoded at the API boundary, which might demand significant engineering effort. Our scope covers relational, API-mediated databases with explicit validation logic;
unstructured content and probabilistic enforcement remain out of scope.
Natural extensions include intent-weight calibration via lightweight environment profiling to close the residual CDF gap and
cross-environment heuristic transfer via tool-signature similarity.

\bibliographystyle{plain}
\bibliography{references}


\appendix


\section{Enterprise Environments}
\label{app:environments}

The domains are deliberately varied in complexity and subject matter: at one end, Forum covers straightforward thread-post-moderation semantics with 10 tools and 10 validation checks; at the other, HR encompasses payroll integrity, organisational hierarchy, time-off entitlements, and staffing policies encoded across 30 tools and 61 validation checks. The schema-inaccessibility constraint is a deliberate design choice, not a limitation. Enterprise schemas encode proprietary domain models and are routinely subject to the same data-governance controls as the data they structure; in many enterprise deployments, schema access requires the same review as production data access. By requiring the synthesizer to operate from the tool registry alone, the framework reflects the actual information boundary available to an external agent in a real enterprise deployment, and the resulting synthesizer operates within the same interface contract as production tool-calling agents.

\begin{table}[h]
  \centering
  \caption{The ten enterprise environments used in this work.
           Each exposes a typed tool registry and a validation suite of
           programmatic business-rule checks.
           The database schema is not accessible to the synthesizer.}
  \label{tab:environments}
  \begin{tabular}{lrrr}
    \toprule
    Domain & Tables & Tools & Validation Checks \\
    \midrule
    HR                    & 71 & 30 & 61 \\
    Airline Reservations  &  9 & 16 & 13 \\
    IT Service Management &  2 & 13 & 10 \\
    Banking               &  3 & 10 & 10 \\
    E-commerce            &  4 & 11 &  9 \\
    Retail                &  8 & 14 & 19 \\
    University            &  4 &  9 & 12 \\
    Pharmacy              &  4 &  9 &  9 \\
    Forum                 &  6 & 10 & 10 \\
    Cloud                 &  6 & 29 &  6 \\
    \bottomrule
  \end{tabular}
\end{table}



\section{Graph Similarity Metrics for Manifest Comparison}
\label{app:graph-similarity}

To quantify how closely the entity dependency graph inferred by GP on the obfuscated airline environment (Figure~\ref{fig:manifest-obfuscated-airline}) matches the graph inferred on the semantic airline environment (Figure~\ref{fig:manifest-airline}), we use two label-independent structural metrics.

\paragraph{Graph Edit Distance (GED).}
Given two graphs $G_1 = (V_1, E_1)$ and $G_2 = (V_2, E_2)$, the \emph{Graph Edit Distance}~\cite{Sanfeliu1983} is the minimum total cost of an edit sequence that transforms $G_1$ into $G_2$:
\begin{equation}
  \mathrm{GED}(G_1, G_2) = \min_{\gamma \,\in\, \Gamma(G_1, G_2)} \sum_{(a \to b)\,\in\,\gamma} c(a \to b),
  \label{eq:ged}
\end{equation}
where $\Gamma(G_1, G_2)$ is the set of all edit paths between $G_1$ and $G_2$, and each edit operation $a \to b$ is one of: insert node, delete node, substitute node label, insert edge, delete edge, or substitute edge label, each with cost $c(\cdot)$.

Since we compare graphs with entirely different node labels (semantic vs.\ opaque identifiers), we set all substitution costs to zero and compare topology only.  We additionally project both directed graphs to their undirected counterparts before computing GED, so that the metric counts only \emph{edge insertions and deletions}:
\begin{equation}
  c(\text{node sub}) = c(\text{edge sub}) = 0, \quad
  c(\text{node ins/del}) = c(\text{edge ins/del}) = 1.
  \label{eq:ged-costs}
\end{equation}
Under these costs, GED equals the size of the symmetric difference of the edge sets of the two unlabeled graphs.  A value of 0 indicates the graphs are structurally identical (isomorphic); each unit increase reflects one additional edge that must be inserted or deleted.

For graphs of the size studied here ($|V| = 11$), exact GED is computed via the branch-and-bound algorithm implemented in NetworkX~\cite{networkx}.

\paragraph{Normalized GED.}
To make GED scale-independent across graphs of different sizes, we normalize by the edge count of the larger graph:
\begin{equation}
  \mathrm{GED}_{\mathrm{norm}}(G_1, G_2) = \frac{\mathrm{GED}(G_1, G_2)}{\max\!\left(|E_1^{u}|,\, |E_2^{u}|\right)},
  \label{eq:ged-norm}
\end{equation}
where $|E_i^{u}|$ denotes the edge count of the undirected projection of $G_i$.  The normalized GED lies in $[0, 1]$: a value of 0 indicates the two graphs are topologically identical, and a value of 1 indicates no structural overlap.

\paragraph{Results.}
Applying these metrics to the airline and obfuscated airline manifest graphs (both $|V|=11$) yields $\mathrm{GED}=2$ and $\mathrm{GED}_{\mathrm{norm}}=0.154$.  The two graphs therefore differ by only 2 edge edits: the obfuscated graph contains edges from the passenger entity (\texttt{object\_d}) to the booking (\texttt{object\_c}), which GP inferred from argument schemas in the absence of semantic cues.



\section{Metric Definitions}
\label{app:metrics}

\subsection{Common Foundation: JSD-Based Similarity}

All three distributional metrics (MDF, PDF, CDF) build on the same foundation.

\textbf{Baseline subtraction.}
Each metric measures only the records \emph{added} by the agent, not rows present in
the pre-run seed snapshot.
For a field value $v$:
\begin{equation}
  c_{\text{agent}}(v) = \max\!\bigl(0,\; c_{\text{live}}(v) - c_{\text{snapshot}}(v)\bigr),
  \label{eq:baseline}
\end{equation}
where $c_{\text{live}}(v)$ is the count of $v$ in the live database after the synthesis run
and $c_{\text{snapshot}}(v)$ is the count in the snapshot taken immediately before.
The $\max(0,\cdot)$ clips seed values that are never overwritten.
This formulation assumes seed rows are not modified or deleted during synthesis, which holds in the STS setting because GP operates exclusively through creation-oriented API workflows; the $\max(0,\cdot)$ guard handles the rare case where a seed row is removed.
For structural metrics (VPR, CDF), row identity is tracked exactly via primary-key set subtraction ($\text{PKs}_{\text{live}} \setminus \text{PKs}_{\text{snapshot}}$), which is immune to modification artefacts.
The observed agent distribution is then $\hat{P}_{\text{agent}}(v) = c_{\text{agent}}(v) / \sum_{v'} c_{\text{agent}}(v')$.

\textbf{JSD-based similarity.}
Let $P$ and $Q$ be discrete distributions over the same finite support $\mathcal{V}$.
The Jensen-Shannon Divergence~\cite{Lin1991jsd} is:
\begin{equation}
  \mathrm{JSD}(P \,\|\, Q) = \tfrac{1}{2}\,\mathrm{KL}(P \,\|\, M) + \tfrac{1}{2}\,\mathrm{KL}(Q \,\|\, M),
  \label{eq:jsd}
\end{equation}
where $M = \tfrac{1}{2}(P+Q)$ is the mixture and $\mathrm{KL}(P\|Q) = \sum_v P(v)\ln\frac{P(v)}{Q(v)}$.
JSD is symmetric, always finite (the mixture $M$ is in the support of both $P$ and $Q$),
and bounded: $0 \leq \mathrm{JSD}(P\|Q) \leq \ln 2$.
We normalise to a similarity in $[0,1]$:
\begin{equation}
  \mathrm{sim}(P, Q) = 1 - \frac{\mathrm{JSD}(P \,\|\, Q)}{\ln 2},
  \label{eq:sim}
\end{equation}
where $\mathrm{sim}=1$ means identical distributions and $\mathrm{sim}=0$ means
maximally divergent (disjoint supports).

\textbf{Bootstrap confidence intervals.}
All three distributional metrics report 95\% CIs via non-parametric bootstrap
resampling ($B=1000$, $\alpha=0.05$).
For a single field with $n$ agent rows: resample $n$ values with replacement,
recompute similarity, repeat $B$ times, report the $[\alpha/2, 1{-}\alpha/2]$ quantiles.
The overall metric score (mean of per-field similarities) uses the same procedure
applied to the vector of per-field point estimates when two or more fields are registered;
this between-field resampling captures heterogeneity across registered columns and
produces narrow CIs only when all per-field similarities are uniformly high, reflecting
genuine distributional agreement rather than an artefact of the resampling procedure.
When only one field is registered for a metric, the bootstrap resamples the underlying
generated rows directly (per-field CI), so the confidence interval correctly reflects
row-level sampling uncertainty rather than returning a degenerate zero-width interval.
A small number of entries in Table~\ref{tab:df-main} report zero-width CIs
(e.g., Forum MDF\,=\,0.945\,[0.945,\,0.945]; HR MDF\,=\,0.939\,[0.939,\,0.939]).
These are not artefacts of the bootstrap procedure: they arise when all generated values
for the registered field collapse to a single category (e.g., all Forum users receive
role \texttt{member}), making the per-field similarity invariant under row-level
resampling. The zero-width CI is itself informative --- it signals that the agent
produces a degenerate, low-entropy distribution for that field, which is reflected in
the corresponding MDF or CDF score (comparison to reference distribution)

\subsection{Marginal Distribution Fidelity (MDF)}

MDF measures per-column realism: does the agent's value distribution for each
registered column match the reference marginal?
For a categorical column $(t,c)$ with reference distribution $P_{\mathrm{ref}}$:
\begin{equation}
  \mathrm{MDF}_{t,c} = \mathrm{sim}(\hat{P}_{\text{agent}},\, P_{\mathrm{ref}}),
  \qquad
  \mathrm{MDF} = \frac{1}{|\mathcal{F}|}\sum_{(t,c)\,\in\,\mathcal{F}} \mathrm{MDF}_{t,c}.
  \label{eq:mdf-app}
\end{equation}
For numeric columns, $\mathrm{MDF}_{t,c} = 1 - D_{n,m}$, where $D_{n,m}$ is the
two-sample Kolmogorov–Smirnov statistic~\cite{Kolmogorov1933,Smirnov1948} between the agent sample and the reference sample.

\subsection{Pairwise Distribution Fidelity (PDF)}

PDF measures joint within-table realism for registered column pairs $(c_1, c_2)$.
Agent and reference joint distributions are computed via the same baseline subtraction
applied to the joint count $c(v_1, v_2)$.
Reference joint distributions satisfy the marginal consistency condition
$\sum_{v_2} P_{\mathrm{ref}}(v_1,v_2) = P_{\mathrm{MDF}}(v_1)$, ensuring PDF
measures correlation structure above and beyond what MDF already captures:
\begin{equation}
  \mathrm{PDF} = \frac{1}{|\mathcal{P}|}\sum_{(t,c_1,c_2)\,\in\,\mathcal{P}}
    \mathrm{sim}(\hat{P}^{\mathrm{joint}}_{\text{agent}},\, P^{\mathrm{joint}}_{\mathrm{ref}}).
  \label{eq:pdf-app}
\end{equation}

\subsection{Cross-table Distribution Fidelity (CDF)}

CDF lifts measurement to the relational structure: for each FK relationship
$(T_p \to T_c, k)$, it compares the distribution of child-record counts per
agent-created parent against a reference.
Let $\mathcal{A}$ be the set of parent PKs created by the agent.
For each $p \in \mathcal{A}$, the child count $n_c(p)$ is bucketed as
$n_c^{\mathrm{bucket}}(p) = \min(n_c(p), n_{\max})$, and the observed distribution
$\hat{P}^{\mathrm{child}}_{\text{agent}}$ is the empirical distribution over bucket values.
\begin{equation}
  \mathrm{CDF} = \frac{1}{|\mathcal{R}|}\sum_{(T_p,T_c,k)\,\in\,\mathcal{R}}
    \mathrm{sim}(\hat{P}^{\mathrm{child}}_{\text{agent}},\, P^{\mathrm{child}}_{\mathrm{ref}}).
  \label{eq:cdf-app}
\end{equation}
Counts exceeding $n_{\max}$ fold into the last bucket to prevent the metric from
becoming undefined when the agent creates more children than the reference anticipates.

\subsection{Average Successful Write Operations (WO)}

Let $\mathcal{W}_e$ be the set of write tools for environment $e$.
For a run of $T$ trajectories, the successful write count for trajectory $t$ is
$w_t = |\{i : f_{t,i} \in \mathcal{W}_e \land \epsilon_{t,i} = \mathrm{false}\}|$,
where $\epsilon_{t,i}$ is the error flag for call $i$.
\begin{equation}
  \mathrm{WO} = \frac{1}{T}\sum_{t=1}^{T} w_t.
  \label{eq:wo-app}
\end{equation}
Only calls that successfully change database state are counted; API-layer rejections
(referential integrity, business rule violations) set $\epsilon_{t,j}=\mathrm{true}$ (where $j$ is tool-call index in $t$) and
are excluded.


\section{Registered Fields and Priors per Environment}
\label{app:priors}

A column (or column pair, or FK relationship) is \emph{registered} for evaluation if and only if its reference distribution can be grounded in publicly available data --- industry reports, government statistics, or established benchmarks --- independently of any data generated by the systems under evaluation.
This criterion keeps reference distributions objective and reproducible: they are fixed before any synthesis run and do not depend on the synthesizer being evaluated.
Columns lacking a publicly verifiable reference prior are excluded from metric computation rather than estimated from seed data, which would create a circularity for methods that use seed data for training.

\subsection*{MDF: Registered marginal columns}

The Table~\ref{tab:mdf-priors} lists the columns tracked for marginal distribution fidelity; each column has a reference categorical or numeric distribution defined.

\begin{table}[t]
\centering
\caption{MDF: registered marginal columns per environment.}
\label{tab:mdf-priors}
\small
\begin{tabular}{llp{7cm}}
\toprule
Environment & Table & Columns \\
\midrule
Airline      & \texttt{reservations}           & \texttt{cabin}, \texttt{total\_baggages}, \texttt{insurance}, \texttt{flight\_type} \\
Banking      & \texttt{account}                & \texttt{account\_type} \\
             & \texttt{txn}                    & \texttt{txn\_type} \\
Cloud        & \texttt{instances}              & \texttt{state}, \texttt{instance\_type}, \texttt{region} \\
             & \texttt{images}                 & \texttt{visibility} \\
Ecommerce    & \texttt{shop\_order}            & \texttt{status} \\
             & \texttt{product}                & \texttt{category} \\
Forum        & \texttt{forum\_user}            & \texttt{role} \\
             & \texttt{thread}                 & \texttt{status} \\
HR           & \texttt{emp\_job}               & \texttt{employment\_status} \\
             & \texttt{tm\_time\_off\_request} & \texttt{time\_off\_type\_code} \\
IT Mgmt      & \texttt{it\_ticket}             & \texttt{priority}, \texttt{status} \\
Pharmacy     & \texttt{prescription}           & \texttt{status} \\
             & \texttt{medication}             & \texttt{classification}, \texttt{form} \\
Retail       & \texttt{orders}                 & \texttt{status} \\
             & \texttt{payment\_methods}       & \texttt{source} \\
             & \texttt{payment\_history}       & \texttt{transaction\_type} \\
University   & \texttt{enrollment}             & \texttt{status} \\
             & \texttt{student}                & \texttt{enrollment\_status} \\
\bottomrule
\end{tabular}
\end{table}

\subsection*{PDF: Registered column pairs}

The Table~\ref{tab:pdf-priors} lists the within-table column pairs tracked for pairwise joint distribution fidelity; each pair captures a semantically meaningful co-occurrence relationship (e.g., cabin class and insurance uptake).

\begin{table}[t]
\centering
\caption{PDF: registered column pairs per environment.}
\label{tab:pdf-priors}
\small
\begin{tabular}{lll}
\toprule
Environment & Table & Column pair \\
\midrule
Airline   & \texttt{reservations}           & \texttt{cabin} $\times$ \texttt{insurance} \\
          & \texttt{reservations}           & \texttt{cabin} $\times$ \texttt{flight\_type} \\
Banking   & \texttt{account}                & \texttt{account\_type} $\times$ \texttt{overdraft\_allowed} \\
Cloud     & \texttt{instances}              & \texttt{state} $\times$ \texttt{region} \\
HR        & \texttt{tm\_time\_off\_request} & \texttt{time\_off\_type\_code} $\times$ \texttt{request\_status} \\
IT Mgmt   & \texttt{it\_ticket}             & \texttt{priority} $\times$ \texttt{status} \\
Retail    & \texttt{orders}                 & \texttt{status} $\times$ \texttt{cancel\_reason} \\
\midrule
\multicolumn{3}{l}{\textit{Ecommerce, Forum, Pharmacy, University: no pairs registered.}} \\
\bottomrule
\end{tabular}
\end{table}

\subsection*{CDF: Registered FK relationships}

The Table~\ref{tab:cdf-priors} lists the parent--child foreign key relationships tracked for cross-table child-count distribution fidelity; reference distributions encode how many child records are expected per parent (e.g., flights and passengers per reservation).

\begin{table}[b]
\centering
\caption{CDF: registered FK relationships per environment.}
\label{tab:cdf-priors}
\small
\begin{tabular}{llll}
\toprule
Environment & Parent & Child & FK \\
\midrule
Airline   & \texttt{reservations}  & \texttt{reservation\_flights}     & \texttt{reservation\_id} \\
          & \texttt{reservations}  & \texttt{reservation\_passengers}  & \texttt{reservation\_id} \\
Banking   & \texttt{account}       & \texttt{txn}                      & \texttt{account\_id} \\
Cloud     & \texttt{vpcs}          & \texttt{subnets}                  & \texttt{vpc\_id} \\
Ecommerce & \texttt{shop\_order}   & \texttt{order\_item}              & \texttt{order\_id} \\
Forum     & \texttt{thread}        & \texttt{post}                     & \texttt{thread\_id} \\
HR        & \texttt{per\_person}   & \texttt{tm\_time\_off\_request}   & \texttt{person\_id} \\
Retail    & \texttt{orders}        & \texttt{order\_items}             & \texttt{order\_id} \\
          & \texttt{orders}        & \texttt{payment\_history}         & \texttt{order\_id} \\
\midrule
\multicolumn{4}{l}{\textit{IT Mgmt, Pharmacy, University: no FK relationships registered.}} \\
\bottomrule
\end{tabular}
\end{table}


\section{Cross-Environment Exploration Budget Analysis}
\label{app:k-analysis}

\subsection*{Airline ablation: full numerical results}

\begin{table}[t]
  \centering
  \caption{Exploration budget ablation on airline ($N\!=\!100$ per condition;
           $K\!=\!100$ approximates $\infty$). Bracketed values are 95\% bootstrap CIs.
           \textbf{Bold}: best per column.}
  \label{tab:k-ablation}
  \footnotesize
  \begin{tabular}{lccccc}
    \toprule
    $K$ & MDF & PDF & CDF & WO\,(mean\,$\pm$\,std) & zero\_frac\,$\downarrow$ \\
    \midrule
    0   & 0.951 [0.876, 0.998] & 0.813 [0.798, 0.827] & \textbf{0.970} [0.941, 0.999] & 0.86\,$\pm$\,0.68 & 0.28 \\
    1   & 0.947 [0.865, 1.000] & 0.837 [0.833, 0.842] & 0.933 [0.885, 0.981] & 0.92\,$\pm$\,0.27 & 0.08 \\
    3   & 0.973 [0.935, 1.000] & 0.900 [0.898, 0.902] & 0.949 [0.923, 0.975] & 1.07\,$\pm$\,0.74 & 0.08 \\
    5   & 0.956 [0.891, 0.998] & 0.868 [0.862, 0.873] & 0.940 [0.902, 0.978] & \textbf{1.18}\,$\pm$\,1.06 & \textbf{0.07} \\
    10  & 0.964 [0.918, 0.998] & 0.891 [0.891, 0.892] & 0.959 [0.945, 0.974] & 1.08\,$\pm$\,0.66 & 0.08 \\
    20  & \textbf{0.975} [0.935, 0.998] & \textbf{0.909} [0.907, 0.910] & 0.925 [0.875, 0.975] & 1.07\,$\pm$\,0.65 & 0.08 \\
    100 & 0.958 [0.900, 0.998] & 0.847 [0.838, 0.855] & 0.934 [0.894, 0.974] & 1.04\,$\pm$\,0.56 & \textbf{0.07} \\
    \bottomrule
  \end{tabular}
\end{table}

\subsection*{Zero-write failure rate ($\mathit{zero\_frac}$) across all environments}

Table~\ref{tab:zerofrac} reports the fraction of zero-write trajectories for all ten
environments under $K\!\in\!\{0, 3, \infty\}$.
Environments with complex multi-step prerequisites (HR, Pharmacy, University)
have persistently elevated failure rates; exploration reduces but does not eliminate
them. Forum is an outlier where $K\!=\!3$ zero\_frac (0.30) exceeds $K\!=\!0$ (0.13):
exploration lead planner towards diversifying the task distribution sampling toward more complex multi-entity workflows; these tasks carry higher failure rates than the simpler tasks selected without prior context, a standard exploration–exploitation tradeoff. 

\begin{table}[t]
  \centering
  \caption{Zero-write trajectory fraction across all environments and exploration
           budgets ($N\!=\!100$ each). Lower is better.
           $K\!=\!\infty$: exploration on every trajectory (pre-fix runs).}
  \label{tab:zerofrac}
  \footnotesize
  \begin{tabular}{lccc}
    \toprule
    Environment & $K\!=\!0$ & $K\!=\!3$ & $K\!=\!\infty$ \\
    \midrule
    Airline       & 0.28 & 0.08 & 0.07 \\
    IT Mgmt       & 0.18 & 0.09 & 0.08 \\
    Forum         & 0.13 & 0.30 & 0.11 \\
    HR            & 0.32 & 0.07 & 0.11 \\
    Banking       & 0.10 & 0.05 & 0.12 \\
    Cloud         & 0.29 & 0.13 & 0.23 \\
    University    & 0.40 & 0.34 & 0.44 \\
    Pharmacy      & 0.33 & 0.22 & 0.16 \\
    Retail        & 0.20 & 0.00 & 0.00 \\
    E-commerce    & 0.20 & 0.02 & 0.02 \\
    \bottomrule
  \end{tabular}
\end{table}

\subsection*{MDF, WO, and $\mathit{zero\_frac}$ across $K$ and all environments}

\begin{table}[b]
  \centering
  \caption{MDF, WO (mean), and zero\_frac across $K\!\in\!\{0, 3, \infty\}$ for all
           ten environments ($N\!=\!100$ each).
           $K\!=\!\infty$: exploration on every trajectory.}
  \label{tab:k-cross-env}
  \footnotesize
  \setlength{\tabcolsep}{4pt}
  \begin{tabular}{l|ccc|ccc|ccc}
    \toprule
    & \multicolumn{3}{c|}{MDF} & \multicolumn{3}{c|}{WO (mean)} & \multicolumn{3}{c}{zero\_frac\,$\downarrow$} \\
    Environment & $K\!=\!0$ & $K\!=\!3$ & $K\!=\!\infty$ & $K\!=\!0$ & $K\!=\!3$ & $K\!=\!\infty$ & $K\!=\!0$ & $K\!=\!3$ & $K\!=\!\infty$ \\
    \midrule
    Airline     & 0.951 & 0.973 & 0.958 & 0.86 & 1.07 & 1.04 & 0.28 & 0.08 & 0.07 \\
    IT Mgmt     & 0.973 & 0.966 & 0.958 & 1.39 & 1.29 & 2.25 & 0.18 & 0.09 & 0.08 \\
    Forum       & 0.948 & 0.945 & 0.991 & 1.73 & 2.27 & 1.01 & 0.13 & 0.30 & 0.11 \\
    HR          & 0.945 & 0.939 & 0.937 & 1.34 & 3.32 & 5.29 & 0.32 & 0.07 & 0.11 \\
    Banking     & 0.939 & 0.864 & 0.802 & 2.27 & 4.59 & 3.13 & 0.10 & 0.05 & 0.12 \\
    Cloud       & 0.848 & 0.850 & 0.840 & 1.50 & 1.85 & 2.76 & 0.29 & 0.13 & 0.23 \\
    University  & 0.780 & 0.809 & 0.763 & 2.18 & 0.75 & 0.73 & 0.40 & 0.34 & 0.44 \\
    Pharmacy    & 0.916 & 0.804 & 0.945 & 1.11 & 1.80 & 1.18 & 0.33 & 0.22 & 0.16 \\
    Retail      & 0.896 & 0.709 & 0.573 & 3.28 & 5.58 & 2.52 & 0.20 & 0.00 & 0.00 \\
    E-commerce  & 0.943 & 0.983 & 0.743 & 2.90 & 2.56 & 3.33 & 0.20 & 0.02 & 0.02 \\
    \bottomrule
  \end{tabular}
\end{table}

Table~\ref{tab:k-cross-env} extends the airline-only ablation (Table~\ref{tab:k-ablation})
to all ten environments, showing that the $K\!=\!3$ operating point generalises.
The dominant pattern is a WO improvement from $K\!=\!0$ to $K\!=\!3$ in environments
with complex multi-step workflows (HR: $+1.98$; Banking: $+2.32$; Retail: $+2.30$)
without proportional MDF degradation.
MDF degradation under higher $K$ is concentrated in lifecycle-ordered environments
where exploration diversifies task types (Retail: $0.896 \to 0.709 \to 0.573$;
Banking: $0.939 \to 0.864 \to 0.802$), confirming the mechanism described in
Section~\ref{sec:ablation}.


\section{Implementation Details}
\label{app:impl}

All agents use Anthropic's \textbf{Claude 4.5 Opus} via a LiteLLM proxy, temperature $1.0$,
and the LangGraph~\cite{LangGraph2024} agentic loop.
GP is evaluated at $K\!=\!3$ exploration trajectories as the primary configuration;
each environment uses $N\!=\!100$ independent trajectories.
All runs start from a clean database with a fixed seed snapshot.
The \textsc{Plan} phase coverage window, the number of recent trajectories scanned to detect unused write operations, is set to $W = \max(10,\; 2 \times |\mathcal{W}|)$, where $|\mathcal{W}|$ is the number of distinct write tools in the environment; this ensures every write operation is exercised at least once per window regardless of environment size.
Statistical synthesizers use SDV~v1.36.1 (CTGAN \& TVAE were trained for 300 epochs, $N\!=\!100$ generated rows per table; HMA uses Gaussian-copula fitting without an epoch budget, $N\!=\!100$ parent rows) and REaLTabFormer~v0.2.4 ($N\!=\!20$ parent rows with children generated proportionally).



\section{Ablation Tables: Heuristic Extraction and Plan Phase}
\label{app:heuristic-ablation}

Table~\ref{tab:heuristic-ablation} isolates the contribution of heuristic extraction
within \textsc{Execute}.
Removing heuristics has the largest effect in Banking, where mean WO drops by 1.49
(33\%) and tool-call overhead per write rises by 54\%: without distilled retry guidance,
the executor repeatedly rediscovers prerequisite orderings from scratch each trajectory.
Airline shows the converse pattern, WO drops slightly ($-$5\%) but
\textit{Calls/WO} nearly doubles ($+$96\%), meaning the agent still completes writes
but wastes far more API calls doing so, rediscovering the shortcut workarounds for each trajectory.
HR is the outlier: heuristics improve WO marginally ($+$7\%) while slightly
increasing calls per write, suggesting that HR's failure modes are driven by
data-dependency complexity that heuristics partially address but do not eliminate.

Table~\ref{tab:plan-ablation} ablates the \textsc{Plan} phase on the Airline environment.
Removing the plan degrades all three distributional metrics uniformly
($\Delta$MDF $-$0.065, $\Delta$PDF $-$0.062, $\Delta$CDF $-$0.071), confirming
that the structured task manifest is necessary for consistent relational coverage.
The most striking effect is on categorical alignment: without a plan, the fraction of
round-trip bookings collapses from 52.5\% to 7.9\% against a reference prior of
55\%, a 44.6 percentage-point gap with plan enabled GP.
The plan phase is therefore the primary mechanism by which GP targets reference
marginals; without it, the executor defaults to whatever task composition emerges
opportunistically from the current DB state.

\begin{table}[t]
  \centering
  \caption{Heuristic extraction ablation ($K\!=\!3$, $N\!=\!100$). Values shown as
           \textit{with-h} / \textit{no-h}; $\Delta$ is no-h $-$ with-h.
           Calls/WO is total tool calls per successful write (lower = more efficient).}
  \label{tab:heuristic-ablation}
  \setlength{\tabcolsep}{5pt}
  \begin{tabular}{lccccc}
    \toprule
    Env     & WO (with / no)  & $\Delta$WO     & Calls/WO (with / no) & $\Delta$Calls/WO  & $\Delta$zero\_frac \\
    \midrule
    Airline & 1.07 / 1.02     & $-$0.05 (5\%)  & 20.5 / 40.0          & $+$19.6 ($+$96\%) & $-$0.02            \\
    Banking & 4.59 / 3.10     & $-$1.49 (33\%) & 2.81 / 4.33          & $+$1.52 ($+$54\%) & $-$0.04            \\
    HR      & 3.32 / 3.09     & $-$0.23 (7\%)  & 8.00 / 7.10          & $-$0.90 ($-$11\%) & $+$0.05            \\
    \bottomrule
  \end{tabular}
\end{table}

\begin{table}[t]
  \centering
  \caption{Plan phase ablation on airline ($K\!=\!3$, $N\!=\!50$). Reference priors:
           \textit{flight\_type} round\_trip\,=\,55\,\%, one\_way\,=\,45\,\%.}
  \label{tab:plan-ablation}
  \setlength{\tabcolsep}{5pt}
  \begin{tabular}{lcccc}
    \toprule
    Condition & MDF & PDF & CDF & round\_trip\,\% \\
    \midrule
    GP $K\!=\!3$ (with plan)   & 0.973 & 0.900 & 0.949 & 52.5\,\% \\
    GP $K\!=\!3$ (no plan)     & 0.908 & 0.838 & 0.878 & \textbf{7.9\,\%}  \\
    $\Delta$                   & $-$0.065 & $-$0.062 & $-$0.071 & $-$44.6\,pp \\
    \bottomrule
  \end{tabular}
\end{table}


\section{Plan Phase Ablation: Trajectory Analysis (Airline)}
\label{app:plan-ablation}

Table~\ref{tab:plan-traj} contrasts trajectory-level statistics between the full
GP ($K\!=\!3$) and the no-plan variant on airline ($N\!=\!50$).

\begin{table}[t]
  \centering
  \caption{Trajectory-level statistics for the plan-phase ablation (airline,
           $K\!=\!3$, $N\!=\!50$). Reference priors: round\_trip\,=\,55\,\%,
           cabin: economy\,=\,75\,\%, basic\_economy\,=\,15\,\%, business\,=\,10\,\%.}
  \label{tab:plan-traj}
  \setlength{\tabcolsep}{5pt}
  \begin{tabular}{lcccc}
    \toprule
    Condition & Bookings/traj & Users/traj & round\_trip\,\% & business\,\% \\
    \midrule
    GP $K\!=\!3$ (with plan) & 0.61  & 1   & 52.5 & 19.7 \\
    GP $K\!=\!3$ (no plan)   & 16.7  & 10  & 7.9  & 22.9 \\
    \bottomrule
  \end{tabular}
\end{table}

\paragraph{Bulk vs.\ single-task execution.}
With planning, each trajectory corresponds to few users performing few actions
(e.g., booking a round-trip business class flight). Without planning, the agent
iterates over all users and books every available flight for each --- 48 of 50
trajectories process all 10 users in one run. No real enterprise API session
operates this way; the pattern resembles a direct database seeding script rather
than simulated user activity.

\paragraph{One-way skew.}
The most damaging distributional error is the collapse of round-trip bookings
from 52.5\,\% to 7.9\,\% against a reference prior of 55\,\%. A round-trip
booking requires composing two flights (outbound and return) into a single
reservation call. Without an explicit plan instruction such as ``book a
round-trip economy flight'', the agent consistently defaults to the structurally
simpler one-way pattern. The PLAN phase resolves this by embedding the target
attribute value directly into the task description.

\paragraph{Concrete trajectory contrast.} We provide example trajectories each when enabling and disabling planner in GP.
\begin{itemize}
    \item \textsc{With plan} --- task: ``Book a round-trip economy class flight departing from JFK for any available gold-tier user; select whatever dates and routes exist for available flight.'' The agent queries users, finds gold\_alice\_1, searches for JFK outbound and return legs, and calls \texttt{book\_reservation} once with \texttt{flight\_type = round\_trip} and \texttt{cabin = economy}. Total tool calls: 18. Write operations: 1. User: 1.
    \item \textsc{Without plan} --- task: ``Populate the database with realistic, diverse data. Use the available write tools to create new records\ldots''The agent lists all users (10), searches available flights for each, then calls \texttt{book\_reservation} 10--22 times in succession (one for every user/flight combination it discovers) all with \texttt{flight\_type = one\_way} because round-trip requires additional flight composition steps the agent does not spontaneously attempt. Total tool calls: 93 (average). Write operations: 10--22. Users: 10.
\end{itemize}



\section{Computational Cost and Throughput}
\label{app:throughput}

Figure~\ref{fig:throughput-latency} shows per-trajectory wall-clock latency and Figure~\ref{fig:throughput-writes} shows effective write throughput for GP ($K\!=\!3$, $N\!=\!100$) across all ten environments.
Table~\ref{tab:throughput} provides the full numerical breakdown including estimated API cost.
Timing is measured from consecutive \texttt{executed\_at} timestamps recorded in each run's exploration manifest.
Cost is estimated using Claude Opus 4.5 rates as of writing this paper.
For comparison, CTGAN fits in 7--25\,s and HMA/REaLTabFormer in 7--33\,s on the same environments, after which rows are generated in under one second.

\begin{figure}[h]
  \centering
  \begin{subfigure}{0.48\linewidth}
    \includegraphics[width=\linewidth]{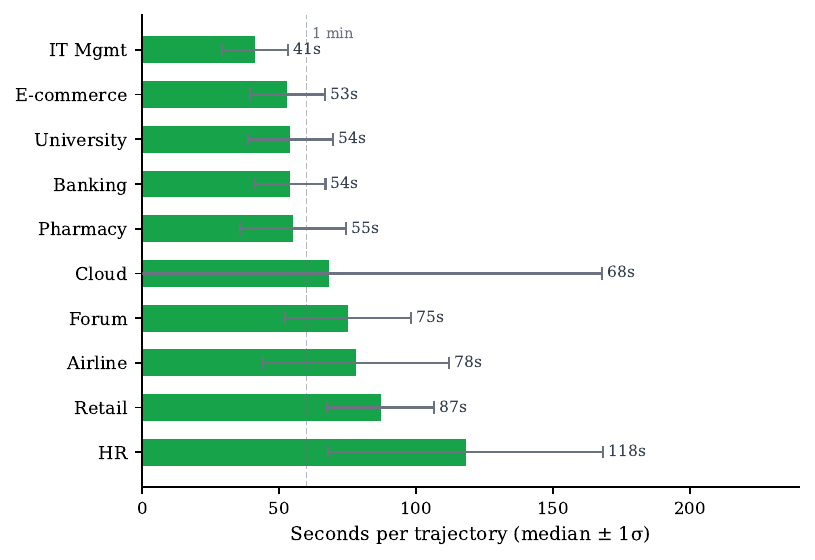}
    \caption{Median per-trajectory latency across all ten environments ($K\!=\!3$, $N\!=\!100$).
             Error bars show $\pm1\sigma$ inter-trajectory variability.
             Environments sorted by median latency.}
    \label{fig:throughput-latency}
  \end{subfigure}
  \hfill
  \begin{subfigure}{0.48\linewidth}
    \includegraphics[width=\linewidth]{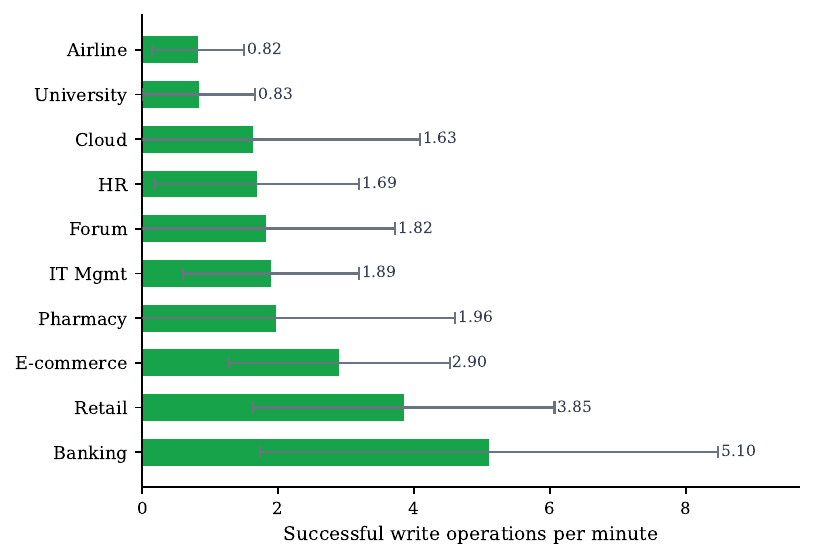}
    \caption{Effective write throughput (WO\,$\div$\,median latency\,$\times\,60$) across all ten environments ($K\!=\!3$, $N\!=\!100$).
             Error bars show propagated uncertainty from WO standard deviation and latency variability.
             Environments sorted by writes per minute.}
    \label{fig:throughput-writes}
  \end{subfigure}
  \caption{GP throughput characterisation across all ten environments.}
  \label{fig:throughput}
\end{figure}



\begin{table}[h]
  \centering
  \caption{GP ($K\!=\!3$, $N\!=\!100$) per-trajectory latency, throughput, and estimated cost across all environments.
           \emph{Writes/min} $=$ WO mean $\div$ median latency $\times\,60$.
           Cost estimated via Claude Opus 4.5 rates \$5/MTok input, \$25/MTok output; blended \$11/MTok.
           Statistical methods (CTGAN, HMA, RTF) fit in 7--33\,s and generate instantaneously
           but are N/A in 7/10 environments due to absent seed data.}
  \label{tab:throughput}
  \footnotesize
  \begin{tabular}{lrrrrrrr}
    \toprule
    Environment & Med.\,s/traj & $\sigma$\,s & Mean\,s & WO/traj & Writes/min & Cost/traj & Cost/write \\
    \midrule
    IT Mgmt       & 41   & 12   & 45   & 1.29 $\pm$ 0.80 & 1.89 & \$0.028 & \$0.022 \\
    E-commerce    & 53   & 14   & 56   & 2.56 $\pm$ 1.27 & 2.90 & \$0.032 & \$0.012 \\
    University    & 54   & 16   & 58   & 0.75 $\pm$ 0.71 & 0.83 & \$0.032 & \$0.042 \\
    Banking       & 54   & 13   & 55   & 4.59 $\pm$ 2.83 & 5.10 & \$0.038 & \$0.008 \\
    Pharmacy      & 55   & 19   & 60   & 1.80 $\pm$ 2.34 & 1.96 & \$0.067 & \$0.037 \\
    Cloud         & 68   & 100  & 87   & 1.85 $\pm$ 1.39 & 1.63 & \$0.027 & \$0.015 \\
    Forum         & 75   & 23   & 79   & 2.27 $\pm$ 2.27 & 1.82 & \$0.172 & \$0.075 \\
    Airline       & 78   & 34   & 86   & 1.07 $\pm$ 0.74 & 0.82 & \$0.083 & \$0.078 \\
    Retail        & 87   & 20   & 90   & 5.58 $\pm$ 2.96 & 3.85 & \$0.128 & \$0.023 \\
    HR            & 118  & 50   & 120  & 3.32 $\pm$ 2.60 & 1.69 & \$0.075 & \$0.023 \\
    \midrule
    CTGAN/HMA   & \multicolumn{3}{c}{7--33\,s (fit only)} & n/a & n/a & n/a & N/A (3/10 envs) \\
    \bottomrule
  \end{tabular}
\end{table}

Several observations are worth noting.
First, environment complexity drives latency: HR (30 tools, 61 checks) takes 2.9$\times$ longer per trajectory than IT Mgmt (13 tools, 10 checks).
The large $\sigma$ for Cloud (100\,s) reflects occasional long-running VPC provisioning trajectories that dominate the tail.
Second, write efficiency and latency are partially decoupled: banking and retail achieve the highest write throughput (5.10 and 3.85 writes/min) despite different latency profiles, because their workflows involve compact multi-step transactions.
Third, the cost per write covers one \emph{fully validated transactional episode}, not one row; an airline write creates 3--6 rows across three tables with all FK constraints satisfied.


\section{Isolating the Simulation Loop: LLM-Direct}
\label{app:llm-direct}

To isolate the contribution of the simulation loop itself, we compare GP against an \emph{LLM-Direct} baseline: the \emph{same} backbone (Claude Opus 4.5) receives the raw schema DDL and seed rows and generates rows that are inserted directly into the database, bypassing the API and its policy enforcement. Table~\ref{tab:llm-direct} reports both methods on airline and HR.

\begin{table}[t]
  \centering
  \caption{LLM-Direct vs.\ STS (GP) on airline and HR ($N\!=\!100$); 95\% bootstrap CIs.
           LLM-Direct inserts rows directly into the database, bypassing the API/policy layer;
           STS generates \emph{through} it with the same backbone. ---: no PDF column pair registered for HR.}
  \label{tab:llm-direct}
  \footnotesize
  \begin{tabular}{llcccc}
    \toprule
    Env. & Method & VPR & MDF & PDF & CDF \\
    \midrule
    \multirow{2}{*}{Airline}
      & LLM-Direct & 0.846 & 0.928 [0.822, 0.998] & 0.813 [0.811, 0.815] & 0.933 [0.892, 0.974] \\
      & STS (GP)   & \textbf{1.000} & 0.973 [0.935, 1.000] & 0.900 [0.898, 0.902] & 0.949 [0.923, 0.975] \\
    \midrule
    \multirow{2}{*}{HR}
      & LLM-Direct & 0.820 & 0.910 [0.884, 0.937] & --- & 0.588 [0.561, 0.613] \\
      & STS (GP)   & \textbf{1.000} & 0.939 [0.939, 0.939] & --- & 0.505 [0.505, 0.505] \\
    \bottomrule
  \end{tabular}
\end{table}

The decisive signal is VPR. Without the API layer, LLM-Direct violates policy constraints in both environments (VPR 0.846 / 0.820), whereas STS attains $\mathrm{VPR}\!=\!1$ by construction. Distributional fidelity is \emph{not} uniformly decisive and we do not claim it as the differentiator: STS is marginally higher on airline, but on HR it actually trails LLM-Direct on CDF (0.505 vs.\ 0.588). This is the one setting where VPR is a competitive result rather than a structural guarantee --- precisely because the baseline bypasses the enforcement layer --- and it shows that the simulation loop contributes exactly what an unconstrained LLM cannot: valid-by-construction data, not merely distribution-matching rows.


\section{Statistical Synthesizer Training Configuration}
\label{app:stat-config}

Statistical synthesizers are applicable only to the three environments that contain
at least one seed row in the pre-run snapshot: \textbf{airline}, \textbf{retail}, and
\textbf{banking}.
The remaining seven environments have zero seed rows and therefore cannot be used to
train any data-driven model; those entries are reported as N/A throughout
Section~\ref{sec:statistical-baselines}.

\textbf{Banking VPR artefact.}
The banking seed snapshot contains customer IDs of the form \texttt{CUST001},
\texttt{CUST002}, \ldots, \texttt{CUST012}.
SDV's CTGAN and TVAE treat these as \emph{incrementable IDs} and re-sample them from
the training pool during synthesis, so generated account rows reference only the
existing twelve customer IDs and trivially satisfy FK-integrity checks.
REaLTabFormer synthesizes fresh account IDs that collide with existing PKs, causing
\texttt{UNIQUE constraint} failures on insert into the snapshot database.
Accordingly, VPR is marked N/A for CTGAN, TVAE, and REaLTabFormer in the banking panel
of Figure~\ref{fig:baselines}.
HMA generates fresh customer rows with novel IDs; its VPR (0.800, 8/10 checks) is
therefore meaningful and is reported.

\textbf{CTGAN and TVAE} (SDV~v1.36.1~\cite{Patki2016sdv}).
Both single-table models are trained independently on each table in the environment.
Training runs for 300 epochs with the default SDV batch size and learning rate.
At inference, $N\!=\!100$ synthetic rows are sampled per table.
Because these models treat each table independently, they have no mechanism for
cross-table relational consistency; CDF is therefore not applicable (N/A).

\textbf{SDV HMA} (SDV~v1.36.1~\cite{Zhang2022sdv}).
HMA is the SDV hierarchical multi-table synthesizer.
It traverses the foreign-key graph top-down, fitting a GaussianCopula per table and
a child-count model per relationship.
Unlike GAN-based synthesizers, HMA employs maximum-likelihood copula fitting without
an epoch budget; $N\!=\!100$ parent rows are sampled at inference.
CDF is applicable because HMA explicitly models parent--child cardinalities.

\textbf{REaLTabFormer} (v0.2.4~\cite{Solatorio2023}).
REaLTabFormer is a relational transformer that autoregressively generates parent rows
and their children.
Due to the model's high per-row computational cost, we generate $N\!=\!20$ parent
rows per environment (with children generated proportionally).
CDF is applicable because child rows are generated conditioned on their parent.

\textbf{Validation.}
All synthesized rows are inserted directly into a clean SQLite database copy via
the \texttt{sqlite3} Python interface, bypassing the API layer.
VPR for statistical synthesizers is measured by applying the same
constraint-check suite used for GP, so a synthesized row that violates a
business rule or referential-integrity constraint is counted as a VPR failure.


\section{EnvScaler vs.\ GP: Trajectory Case Study (Airline)}
\label{app:envscaler-example}

This appendix contrasts a representative EnvScaler trajectory against a representative
GP trajectory on the airline environment to make concrete why schema-privileged task
composition leads to brittle execution.

\subsection*{EnvScaler: multi-subtask task that produces zero writes}

EnvScaler's task generator uses the full database schema and live entity IDs to compose
highly specific, multi-step tasks.
A typical generated task reads:

\begin{quote}
\small
\textit{Gold member Alice Johnson needs several updates. First, cancel her active business
class reservation \texttt{ZFA001}, ensuring all cancellation conditions are satisfied and
refunds are processed to her original payment methods.
Next, for her active round-trip economy reservation \texttt{DF3303}, change both flights to
the next available direct flights on the same route and date, paying or refunding any fare
difference.
After updating the flights, increase the total baggages by 1.
Then, replace the passenger list on \texttt{DF3303} with a new single passenger:
``Emily Johnson'', DOB 2000-09-15.
Finally, if Alice holds more than one active certificate payment method, remove all but
one, then issue her a new travel certificate worth \$150.}
\end{quote}

Because all the EnvScaler tasks are generated independently, two different tasks can have overlapping destructive write operations out of which only one should be possible. Thus in our current example since the previous task already deleted the defined reservation, the agent's first tool call is \texttt{cancel\_reservation} on \texttt{ZFA001} fails immediately: reservation \texttt{ZFA001} has already been cancelled in
the live database (it is present in the seed snapshot but no longer active at execution time).
Because the task is written around specific entity IDs read from the schema at task-generation
time, the agent has no fallback --- the remainder of the trajectory consists of two read calls
(\texttt{get\_reservation\_details}, \texttt{get\_user\_details}) and terminates with
zero net writes.

This failure mode is systematic: 82 of 100 EnvScaler trajectories produce zero net
writes, because any mismatch between the entity state assumed at task-generation time
and the live database state at execution time causes hard failures on the first write
call.

\subsection*{GP: focused single-intent task that succeeds}

GP's Plan phase generates a focused, entity-agnostic task description:

\begin{quote}
\small
\textit{Find any user with available payment methods (credit card or gift cards).
Search for available round-trip economy class flights between any two airports for
upcoming dates.
Book a reservation for 2 passengers without insurance.
After booking, add 1 additional checked bag to the reservation.}
\end{quote}

Because the task references no specific entity IDs, in the \textsc{Execute} phase, GP begins by
querying live state (\texttt{list\_users}, \texttt{get\_user\_details},
\texttt{search\_direct\_flight}) to discover valid entities at execution time.
It then calls \texttt{book\_reservation} and \texttt{update\_reservation\_baggages},
both of which succeed, producing 4 successful write operations.

\subsection*{Summary}

The contrast illustrates a general principle: \emph{schema knowledge enables richer
task specification but couples task prompts to specific entity states}.
In environments where entity state evolves over the synthesis run (reservations are
created, modified, and cancelled by prior trajectories), hard-coded entity IDs become
stale.
GP's entity-agnostic task design is inherently resilient: it discovers valid entities
at execution time rather than assuming them at planning time.


\section{Real-World Enterprise Generalization: Details}
\label{app:real-world}

This appendix provides the system descriptions, run provenance, and per-metric detail
behind Table~\ref{tab:real-world} (Section~\ref{sec:real-world}).

\paragraph{Systems.}
The \textbf{Payroll} system is a proprietary, multi-module enterprise payroll application
(42 registered tools over 13 tables) covering employee compensation, pay runs, deductions,
and disbursement records. Because it is proprietary, it is used as an evaluation target
only and is not part of the released artifacts; no production records from it are used,
retained, or released. \textbf{MintHCM}~\cite{MintHCM} is an open-source Human Capital
Management platform (54 registered tools over 23 tables) covering employees, recruitment
(job positions and candidatures), and organizational structure. For both systems, GP was
given only the tool repository: no schema, no business-rule encodings, no human-authored
write-operation distributions, and no per-system configuration.

\paragraph{Run provenance.}
Both evaluations were run with GP at $K\!=\!3$, $N\!=\!100$ (the paper's primary
configuration), using the Claude Opus 4.5 backbone. The Payroll run was executed on
2026-07-25 and the MintHCM run on 2026-07-26, each from a clean database with a fixed seed
snapshot. The MintHCM source was pinned to the \texttt{master} branch commit checked out at
run time. Both systems attain $\mathrm{VPR}\!=\!1.0$: every trajectory accepted by the system's own API yields a database state that passes that system's validation suite, confirming that the structural guarantee of Proposition~\ref{prop:vpr} is a property of the paradigm and holds on systems GP was never built against.

\paragraph{Interpreting the weaker MintHCM distributional scores.}
MintHCM's PDF (0.344) is dominated by a single high-cardinality registered pair --- the
\texttt{candidatures} joint of recruitment \texttt{status} and application \texttt{source}.
At $N\!=\!100$ this joint is sparsely populated relative to its support, so the empirical
estimate is governed by sampling noise rather than a modeling failure; the effect is the
same data-volume phenomenon analyzed for the authored environments in
Section~\ref{sec:ablation}, amplified here by the larger support of the MintHCM pair.


\section{Backbone Sensitivity: Full Confidence Intervals}
\label{app:backbone}

Table~\ref{tab:backbone-ci} reports the full 95\% bootstrap confidence intervals
($B\!=\!1000$) and record counts behind the point estimates in Table~\ref{tab:backbone}
(Section~\ref{sec:backbone}). Every backbone--environment pair attains $\mathrm{VPR}\!=\!1.0$.

\begin{table}[t]
  \tiny
  \centering
  \setlength{\tabcolsep}{2pt}
  \caption{Backbone sensitivity, full detail: GP ($K\!=\!3$, $N\!=\!100$) with three LLM
           backbones across three environments. Values are point estimates with 95\%
           bootstrap CIs ($B\!=\!1000$). \emph{Records}: new rows generated.
           ---: no column pairs registered for HR (PDF).}
  \label{tab:backbone-ci}
  \scriptsize
  \begin{tabular}{llcccrr}
    \toprule
    Environment & Backbone & MDF & PDF & CDF & \multicolumn{1}{c}{WO} & \multicolumn{1}{c}{Records} \\
    \midrule
    \multirow{3}{*}{Airline}
      & Claude Opus 4.5     & 0.973 [0.935, 1.000] & 0.900 [0.898, 0.902] & 0.949 [0.923, 0.975] & 1.07\,$\pm$\,0.74 & 389 \\
      & Gemini 2.5 Pro      & 0.937 [0.828, 0.991] & 0.685 [0.685, 0.685] & 0.957 [0.916, 0.999] & 0.68\,$\pm$\,0.83 & 137 \\
      & Qwen3.6-35B-A3B-FP8 & 0.904 [0.842, 1.000] & 0.788 [0.759, 0.817] & 0.980 [0.976, 0.983] & 2.05\,$\pm$\,1.84 & 457 \\
    \midrule
    \multirow{3}{*}{HR}
      & Claude Opus 4.5     & 0.939 [0.939, 0.939] & --- & 0.505 [0.505, 0.505] & 3.32\,$\pm$\,2.60 & 756 \\
      & Gemini 2.5 Pro      & 0.940 [0.940, 0.940] & --- & 0.514 [0.514, 0.514] & 2.72\,$\pm$\,3.25 & 436 \\
      & Qwen3.6-35B-A3B-FP8 & 0.937 [0.937, 0.937] & --- & 0.509 [0.509, 0.509] & 8.91\,$\pm$\,11.65 & 1087 \\
    \midrule
    \multirow{3}{*}{Cloud}
      & Claude Opus 4.5     & 0.850 [0.707, 0.943] & 0.825 [0.765, 0.847] & 0.610 [0.311, 0.757] & 1.85\,$\pm$\,1.39 & 102 \\
      & Gemini 2.5 Pro      & 0.831 [0.678, 0.923] & 0.838 [0.751, 0.859] & 0.793 [0.333, 1.000] & 10.60\,$\pm$\,48.29 & 120 \\
      & Qwen3.6-35B-A3B-FP8 & 0.939 [0.903, 0.970] & 0.837 [0.784, 0.853] & 0.504 [0.298, 0.639] & 2.55\,$\pm$\,4.71 & 114 \\
    \bottomrule
  \end{tabular}
\end{table}

The Gemini 2.5 Pro Cloud WO mean (10.60) carries a large standard deviation (48.29): the
mean is skewed by a small number of long trajectories, consistent with high-variance
exploration rather than instability (the run still attains $\mathrm{VPR}\!=\!1.0$ and
produces 120 valid records). Across all nine backbone--environment pairs the quantitative
variation is comparable in magnitude to the cross-environment spread already present in
Table~\ref{tab:df-main}, and reflects differences in model pretraining rather than a
structural fragility of STS.


\section{Broader Impacts}
\label{app:broader-impacts}

STS could be used to produce training data for agents that operate against enterprise systems without authorization. Here the marginal risk is bounded by the structure of the method. STS requires write access to a live, policy-enforcing API to generate data; any actor with that access already has at least read access to the system's records, so STS confers no data-access capability the actor did not already possess. Moreover, because every accepted trajectory is a \emph{valid}, policy-compliant interaction (by Proposition~\ref{prop:vpr}), training on STS-generated data teaches policy-compliant behavior rather than exploitation or policy evasion. The capability STS adds over the status quo is therefore the convenience of generating realistic valid interaction data, not new access or new attack surface.

\newpage
\section*{NeurIPS Paper Checklist}

\begin{enumerate}

\item {\bf Claims}
    \item[] Question: Do the main claims made in the abstract and introduction accurately reflect the paper's contributions and scope?
    \item[] Answer: \answerYes{}
    \item[] Justification: The abstract and Section~1 state four concrete, verifiable contributions (STS paradigm, GP agent, ten-environment evaluation, open-source release). All claims are directly supported by experimental results in Section~5 and the appendices. Aspirational goals (e.g., future cross-model generalisation) are explicitly mentioned as limitations or future work.

\item {\bf Limitations}
    \item[] Question: Does the paper discuss the limitations of the work performed by the authors?
    \item[] Answer: \answerYes{}
    \item[] Justification: A dedicated paragraph in Section~6 discusses limitations.

\item {\bf Theory assumptions and proofs}
    \item[] Question: For each theoretical result, does the paper provide the full set of assumptions and a complete (and correct) proof?
    \item[] Answer: \answerYes{}
    \item[] Justification: The paper contains one formal result, Proposition~1 (Structural Validity by Construction, Section~3.2), which states that GP-generated databases satisfy all validation checks by construction. The full proof is provided immediately following the proposition.

\item {\bf Experimental result reproducibility}
    \item[] Question: Does the paper fully disclose all the information needed to reproduce the main experimental results of the paper to the extent that it affects the main claims and/or conclusions of the paper (regardless of whether the code and data are provided or not)?
    \item[] Answer: \answerYes{}
    \item[] Justification: Appendix~\ref{app:impl} specifies the system configuration, Section~5.1 describes the experimental setup. Reference distributions for all environments are provided in Appendix~\ref{app:priors}.

\item {\bf Open access to data and code}
    \item[] Question: Does the paper provide open access to the data and code, with sufficient instructions to faithfully reproduce the main experimental results, as described in supplemental material?
    \item[] Answer: \answerYes{}
    \item[] Justification: The full framework, all ten environments, and generated datasets are open-sourced.

\item {\bf Experimental setting/details}
    \item[] Question: Does the paper specify all the training and test details (e.g., data splits, hyperparameters, how they were chosen, type of optimizer) necessary to understand the results?
    \item[] Answer: \answerYes{}
    \item[] Justification: Section~5.1 and Appendix~\ref{app:impl} specify all relevant implementation details. GP has no training phase; synthesizer hyperparameters (exploration budget $K$, number of trajectories $N$, temperature) are stated explicitly.

\item {\bf Experiment statistical significance}
    \item[] Question: Does the paper report error bars suitably and correctly defined or other appropriate information about the statistical significance of the experiments?
    \item[] Answer: \answerYes{}
    \item[] Justification: All distributional fidelity metrics in Tables~\ref{tab:envscaler} and~\ref{tab:k-ablation} report 95\% bootstrap confidence intervals ($B\!=\!1000$ resamples); the bootstrap procedure is described in Appendix~\ref{app:metrics}.

\item {\bf Experiments compute resources}
    \item[] Question: For each experiment, does the paper provide sufficient information on the computer resources (type of compute workers, memory, time of execution) needed to reproduce the experiments?
    \item[] Answer: \answerYes{}
    \item[] Justification: All LLM calls are routed through a hosted API (\texttt{claude-4.5-opus} via LiteLLM proxy); no GPU compute is required for generating data. Appendix~\ref{app:impl} specifies the model, framework, and per-condition batch sizes sufficient for others to estimate reproduction cost. But we also performed downstream agent training, shown in Section~\ref{sec:downstream}, for which we used an 8 A100 40GB compute node.

\item {\bf Code of ethics}
    \item[] Question: Does the research conducted in the paper conform, in every respect, with the NeurIPS Code of Ethics?
    \item[] Answer: \answerYes{}
    \item[] Justification: The research involves no human subjects, no personal data collection, and no scraping of third-party content. All environments use entirely artificial seed data. The paper complies with the NeurIPS anonymity requirements for the submission phase.

\item {\bf Broader impacts}
    \item[] Question: Does the paper discuss both potential positive societal impacts and negative societal impacts of the work performed?
    \item[] Answer: \answerYes{}
    \item[] Justification: Appendix~\ref{app:broader-impacts} discusses impacts (privacy-preserving research infrastructure).

\item {\bf Safeguards}
    \item[] Question: Does the paper describe safeguards that have been put in place for responsible release of data or models that have a high risk for misuse?
    \item[] Answer: \answerNA{}
    \item[] Justification: The paper releases a data synthesis framework and ten simulated environments, not pretrained language models or scraped internet data. The released assets pose no high risk for misuse beyond standard software release concerns.

\item {\bf Licenses for existing assets}
    \item[] Question: Are the creators or original owners of assets used in the paper properly credited and are the license and terms of use explicitly mentioned?
    \item[] Answer: \answerYes{}
    \item[] Justification: All third-party packages are cited: SDV~\cite{Patki2016sdv}, REaLTabFormer~\cite{Solatorio2023}, LangGraph~\cite{LangGraph2024}. The supplementary material lists the licence for each dependency (SDV: Business Source License; REaLTabFormer: MIT; LangGraph: MIT).

\item {\bf New assets}
    \item[] Question: Are new assets introduced in the paper well documented and is the documentation provided alongside the assets?
    \item[] Answer: \answerYes{}
    \item[] Justification: The ten simulated enterprise environments and generated datasets are released with the anonymised supplementary code. It includes README describing the step by step method for usage. Dataset documentation follows the standard datasheet format in the supplementary material.
    \begin{itemize}
        \item The answer \answerNA{} means that the paper does not release new assets.
        \item Researchers should communicate the details of the dataset\slash code\slash model as part of their submissions via structured templates. This includes details about training, license, limitations, etc. 
        \item The paper should discuss whether and how consent was obtained from people whose asset is used.
        \item At submission time, remember to anonymize your assets (if applicable). You can either create an anonymized URL or include an anonymized zip file.
    \end{itemize}

\item {\bf Crowdsourcing and research with human subjects}
    \item[] Question: For crowdsourcing experiments and research with human subjects, does the paper include the full text of instructions given to participants and screenshots, if applicable?
    \item[] Answer: \answerNA{}
    \item[] Justification: The paper involves no crowdsourcing and no research with human subjects.

\item {\bf Institutional review board (IRB) approvals or equivalent for research with human subjects}
    \item[] Question: Does the paper describe potential risks incurred by study participants and whether IRB approvals were obtained?
    \item[] Answer: \answerNA{}
    \item[] Justification: No human subjects are involved in this research.

\item {\bf Declaration of LLM usage}
    \item[] Question: Does the paper describe the usage of LLMs if it is an important, original, or non-standard component of the core methods in this research?
    \item[] Answer: \answerYes{}
    \item[] Justification: LLM inference is the core computational component of the paper. Section~4 describes how \texttt{claude-4.5-opus} is used in the Explore, Plan, and Execute phases of GP, and Appendix~\ref{app:impl} specifies the exact model version, temperature, and framework.

\end{enumerate}

\end{document}